\documentclass[letter,11pt]{article}
\usepackage[margin=1in]{geometry}
\usepackage{amsmath,graphicx,multicol}
\usepackage{amsfonts}
\usepackage{color,xcolor}
\usepackage{subcaption}
\usepackage{bbm}
\usepackage{cite}
\usepackage{amssymb}
\usepackage{tabularx,booktabs}
\usepackage{algorithm}
\usepackage{algcompatible}
\usepackage[utf8]{inputenc}
\usepackage[T1]{fontenc}
\usepackage[english]{babel}
\usepackage{enumitem}

\usepackage{amsthm}
\usepackage{bm}
\usepackage{url,hyperref}
\usepackage{cleveref}
\usepackage{mathtools}
\usepackage[bb=dsserif]{mathalpha}
\newtheorem{definition}{Definition}[]

\newtheorem{theorem}{Theorem}[]

\newtheorem{lemma}[]{Lemma}

\newcommand{\argmax}{\arg\!\max}
\newcommand{\argmin}{\arg\!\min}

\DeclarePairedDelimiterX{\infdivx}[2]{(}{)}{%
  #1\;\delimsize\|\;#2%
}
\newcommand{\infdiv}{D_{\text{KL}}\infdivx}
\DeclareMathOperator{\E}{\mathbb{E}}
\DeclarePairedDelimiterX{\gendivx}[2]{(}{)}{%
  #1\;\delimsize\|\;#2%
}
\newcommand{\gendiv}{D\gendivx}

\def\R{{\mathbb{R}}}

\def\lp{\left(}
\def\rp{\right)}

\def\E{\mathbb{E}}
\def\R{\mathbb{R}}

\def\sS{\mathcal{S}}
\def\sN{\mathcal{N}}
\def\sT{\mathcal{T}}

\newcommand{\RNum}[1]{\uppercase\expandafter{\romannumeral #1\relax}}
\makeatletter
\renewcommand{\fnum@figure}{Fig.~\thefigure}
\makeatother

\title{\textbf{Localized Distributional Robustness in Submodular Multi-Task Subset Selection}}
\date{}
\author{Ege C. Kaya, Abolfazl Hashemi\thanks{The authors are with the Elmore Family School of Electrical and Computer Engineering, Purdue University, West Lafayette, IN, USA. This work was presented in part at the 2023 Annual Conference on Communication, Control, and Computing (Allerton) \cite{10313462}. The full work was published in IEEE Transactions on Signal Processing, 2024 \cite{TSP}. This work is supported in part by DARPA grant HR00112220025.}}
\begin{document}
\maketitle
%%%%%%%%%%%%%%%%%%%%%%%%%%%%%%%%%%%%%%%%%%%%%%%%%%%%%%%%%%%%%
%%%%%%%%%%%%%%%%%%%%%%%%%%%%%%%%%%%%%%%%%%%%%%%%%%%%%%%%%%%%%
\begin{abstract}
In this work, we treat the problem of multi-task submodular optimization from the perspective of local distributional robustness within the neighborhood of a reference distribution which assigns an importance score to each task. We initially propose to introduce a relative-entropy regularization term to the standard multi-task objective. We then demonstrate through duality that this novel formulation 
itself is equivalent to the maximization of a monotone increasing function composed with a submodular function, which may be efficiently carried out through standard greedy selection methods. This approach bridges the existing gap in the optimization of performance-robustness trade-offs in multi-task subset selection. To numerically validate our theoretical results, we test the proposed method in two different settings, one on the selection of satellites in low Earth orbit constellations in the context of a sensor selection problem involving weak-submodular functions, and the other on an image summarization task using neural networks involving submodular functions. Our method is compared with two other algorithms focused on optimizing the performance of the worst-case task, and on directly optimizing the performance on the reference distribution itself. We conclude that our novel formulation produces a solution that is locally distributional robust, and computationally inexpensive.
\end{abstract}
%%%%%%%%%%%%%%%%%%%%%%%%%%%%%%%%%%%%%%%%%%%%%%%%%%%%%%%%%%%%%
%%%%%%%%%%%%%%%%%%%%%%%%%%%%%%%%%%%%%%%%%%%%%%%%%%%%%%%%%%%%%
\section{Introduction}\label{sec:intro}
Submodular functions have for long been the focus of extensive studies\cite{edmonds2003submodular, fujishige2005submodular, iyer2013submodular, krause2014submodular}, thanks to their propensity to occur naturally in many distinct domains such as economics and algorithmic game theory \cite{cardinality, revenue, gametheory} and machine learning\cite{MLRef1,MLRef2,MLRef3,MLRef4,MLRef5,MLRef6,MLRef7,MLRef8,MLRef9,MLRef10}, within many distinct problems, ranging from sensor selection \cite{RSOS, krause2014submodular, greedysensor, krausenearoptimal} to social network modeling \cite{Mossel2007}. In simple terms, a submodular function is a set function that possesses the \textit{diminishing marginal returns} property \cite{nemhauser1978analysis}. This property, at first glance, enables an analogy between submodular functions and concave functions in the continuous domain. However, interestingly, a similar analogy may be drawn between submodular functions and \textit{convex} functions as well, through the use of certain extensions of submodular functions onto the continuous domain \cite{Lovasz1983, krause2014submodular}. The similarity of submodular functions to both convex and concave functions, the white whales of optimization literature, has also garnered interest in the study of the optimization of submodular functions.

%In this regard, it is frequently compared with and likened to a discrete analogue of concave functions in optimization literature. While this is correct in some sense, peculiarly enough, submodular functions may be seen to behave similarly to convex functions as well.\cite{krause2014submodular}

%\IEEEPARstart{S}{ubmodularity} is a well-studied property of set functions. A submodular function is a set function that satisfies the renowned \textit{diminishing marginal gains} property \cite{nemhauser1978analysis}. In this regard, it is frequently compared with and likened to a discrete analogue of concave functions in optimization literature. While this is correct in some sense, peculiarly enough, submodular functions may be seen to behave similarly to convex functions as well.\cite{krause2014submodular}

%Thanks to their similarities to both concave and convex functions, and their frequency of occurring in numerous natural settings, such as in sensor selection problems \cite{RSOS, krause2014submodular, greedysensor, krausenearoptimal}, economics and algorithmic game theory \cite{cardinality, revenue, gametheory}, the study of the optimization of submodular functions has garnered significant interest \cite{edmonds2003submodular, fujishige2005submodular, iyer2013submodular, krause2014submodular}. 

The quintessential submodular optimization problem is that of the maximization of a submodular function under a cardinality constraint. This problem entails selecting a subset out of a ground set $\mathcal{N}$, such that the size of the selection does not exceed a predetermined constraint value, and the value of the selection under a submodular function is as high as possible. The evaluation $f(\mathcal{S})$ of a subset $\mathcal{S} \subseteq \mathcal{N}$ is commonly called the \textit{score} or \textit{utility of the set} $\sS$. The following formulation, where $f$ is a submodular function, $\sN$ is the ground set and $K$ is a positive integer, formalizes this problem\cite{cardinality}:
%One classical setting often studied in submodular optimization literature is the problem of maximizing a submodular function in the presence of a cardinality constraint \cite{asadpour2008stochastic ,cardinality, krause2014submodular}. Here, the task is to algorithmically make an optimal selection of elements from a universal set, while making sure that the number of elements in the selected subset does not exceed a given integer value. This problem can be concisely expressed as follows \cite{cardinality}:
\begin{equation}
\begin{gathered}
\max_{\sS\subseteq \sN}f(\sS) \label{cardinality}\\
\text{s.t.}\;\lvert \sS\rvert \leq K.
\end{gathered}
\end{equation}
%Here, $f$ is a submodular function, $N$ is the universal set and $K$ is a positive integer.

The notion of \textit{robustness} occurs very frequently in optimization literature and is the object of much consideration. Although there are differing definitions of robustness depending on the context, such as being immune to removals from the solution \cite{partitioning}, or to slight changes in parameters or input data\cite{Lan2020}, and others\cite{robust-survey}, it always encompasses the general idea of a produced output to a problem staying valid under perturbations to the input. In the context of submodular optimization, one notion of robustness that is commonly encountered is \textit{multi-task robustness}. As evident by the name, with this notion, the aim is to algorithmically produce a solution set whose utility with respect to a family of $n$ submodular functions $f_1, \ldots, f_n$ is satisfactory.

As it stands, the notion of ``satisfactory'' is intentionally unspecified, and the problem can be formulated in multiple ways to fit the description. A straightforward interpretation with this goal in mind might lead to what we will call the \textit{worst-case formulation} \cite{krause2014submodular, ben2009robust, ben2002robust, NEURIPS2018_7e448ed9}, which aims to produce a solution that maximizes the worst-performing objective function among all. Using the shorthand notation $[n] \vcentcolon= \{1, \ldots, n\}$, this formulation is formalized:
%A concept that is ever-present in optimization literature is the concept of \textit{robustness}. The exact meaning of robustness is dependent on the context of the problem at hand, various notions have been proposed in the literature, e.g. being immune to removals from the solution set\cite{partitioning}, to slight changes in parameters or input data\cite{Lan2020}, and others\cite{robust-survey}. In submodular optimization literature, one frequent understanding of robustness is that of being robust in the presence of multiple objective functions. In this context, a robust algorithm would aim to produce a solution that attains a high value when evaluated at each of the objective functions. A most straightforward formulation with this goal in mind might be \cite{krause2014submodular, ben2009robust, ben2002robust, NEURIPS2018_7e448ed9}
\begin{equation}
\begin{gathered}
\max_{\sS \subseteq \sN} \min_{i \in [n]} f_i(\sS) \label{robust-first}\\
\text{s.t.}\; \lvert \sS\rvert \leq K.
\end{gathered}
\end{equation}
A shortcoming of this formulation is that it is \textit{too pessimistic}\cite{malherbe2022robustness}, dedicating all resources to the maximization of the worst-performing objective function, at the expense of the others. In a scenario where one function is a clear straggler, always scoring significantly lower than all others, this approach may effectively lead to getting low utility on all functions in pursuit of the hopeless aim of maximizing the straggler.
%This is evidently a simple extension of Problem \eqref{cardinality} under the scope of the robustness consideration, where we have $n$ objective functions. However, one might argue that this formulation is too pessimistic, focusing on optimizing the single worst-case task. Indeed, this might be the best option when we lack any further information on the structure of our problem.

Another naive way to formulate the multi-task robust problem is to adopt what we will call the \textit{average-case formulation} \cite{malherbe2022robustness, pmlr-v89-staib19a}, which aims to directly optimize the mean of all the objective functions. In this case, the formulation would be

%Another standard approach would be to aim to optimize the average case performance of our solution \cite{malherbe2022robustness, pmlr-v89-staib19a}, i.e.,
\begin{equation}
\begin{gathered}
\max_{\sS \subseteq \sN} \frac{1}{n}\sum_{i=1}^n f_i(\sS) \label{robust-avg}\\
\text{s.t.}\;\lvert \sS\rvert \leq K.
\end{gathered}
\end{equation}
The apparent shortcoming of this formulation is that, contrary to the previous formulation, it provides no guarantees about how bad the utility with respect to any single objective function may be. One or more objective functions could be scoring arbitrarily badly, as long as the other objective functions are scoring well enough to compensate for the underperformers. 

The previous two formulations, even with their stated shortcomings, may very well be acceptable approaches, especially in the absence of any additional information. However, we argue that if we possess additional information on the relative importance of the objective functions, we can progress to a more meaningful formulation. Suppose, for instance, that we have access to a \textit{reference distribution}, a discrete probability distribution $Q\in\Delta_n$, where $\Delta_n$ indicates the $n$-dimensional simplex. In a practical setting, the reference distribution $Q$ may be obtained by a decision-maker subjectively assigning importance scores to each objective function, or by a frequentist approach, where the value assigned to each function would be inferred from how frequently the task that is modeled by said function is performed by the system. It is reasonable to argue that the reference distribution can facilitate incorporating the intent and priority of the main stakeholders in the problem. The incorporation of this additional information to the previous formulation leads to a generalization of \eqref{robust-avg}, which we can view as the scenario where the reference distribution is assumed to be uniform. This idea is encapsulated in the following formulation:
\begin{equation}
\begin{gathered}
\max_{\sS \subseteq \sN} \sum_{i=1}^n Q_if_i(\sS) \label{robust-weighted-avg}\\
\text{s.t.}\;\lvert \sS\rvert \leq K.
\end{gathered}
\end{equation}
We note that this formulation is a also generalization of \eqref{robust-first}, in that when $Q$ is a one-hot distribution assigning all weight to the worst-performing objective function, this problem becomes equivalent to \eqref{robust-first}.
\begin{figure}
\centering\includegraphics[width=.3\linewidth]{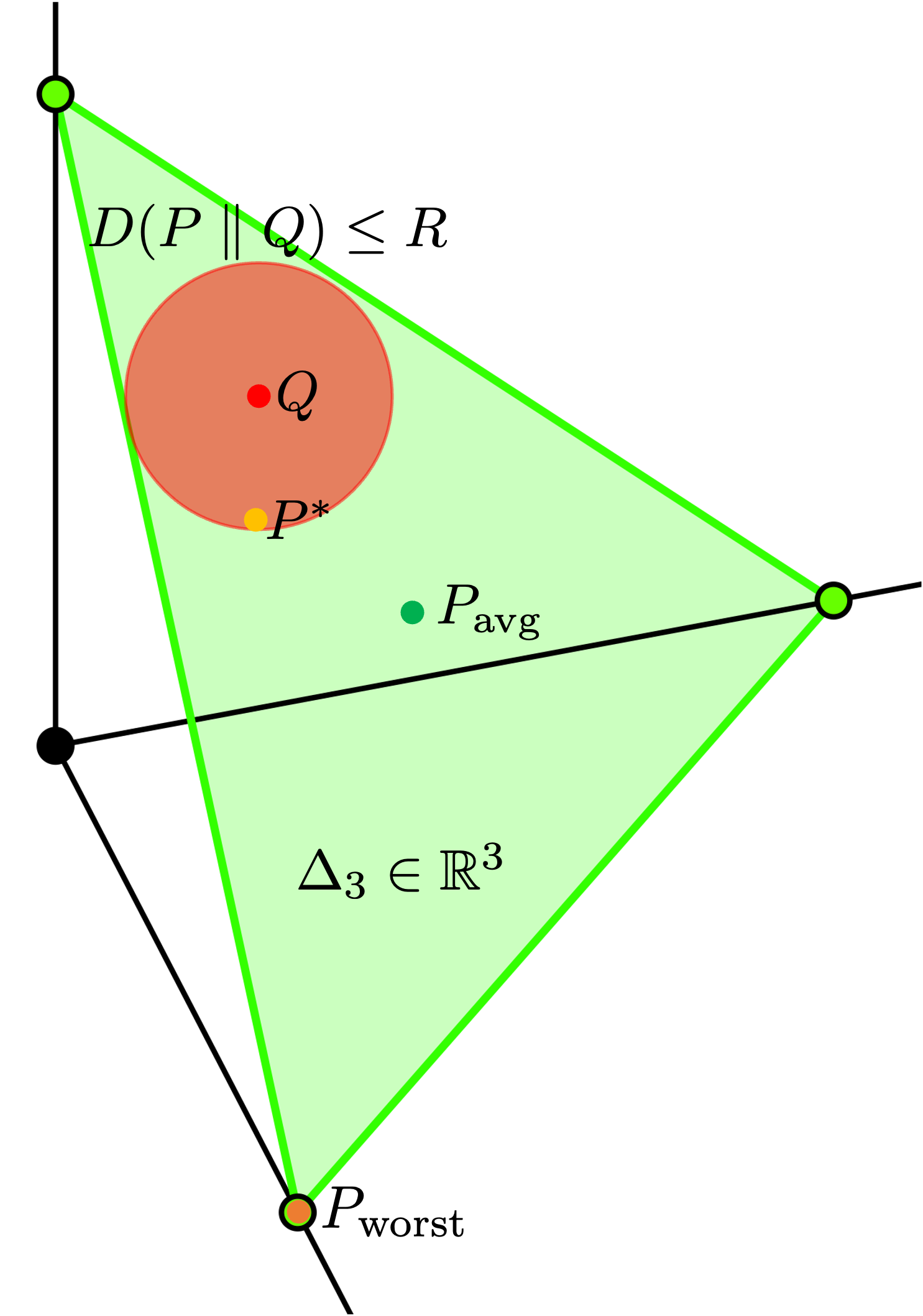}
  \caption{Illustration of the three discrete distributions on the $3$-dimensional simplex\cite{simpleximage} that the three discussed approaches optimize. $P_{\text{worst}}$ corresponds to the global worst-case task scenario, assigning a weight of $1$ to the worst-case task and a weight of $0$ to all the others, residing on a vertex of the simplex. $P_{\text{avg}}$ assigns uniform weight to all tasks, and lies in the center of the simplex. $Q$ is the reference distribution within a neighborhood of which we want to achieve local robustness. $P^\ast$ is the local worst-case distribution within that neighborhood of $Q$.}
  \label{fig:simplex}
\end{figure}
%However, this may achieve arguably little in the sense of robustness, seeing that the performance of the solution with respect to individual objective functions might still be arbitrarily low, even when the average-case performance is satisfactory.
\subsection{Contribution}
%In this work, we assume knowledge of extra information on the structure of the problem. Suppose, as in Problem \eqref{robust-first}, that we aim to be robust in the presence of $n$ objective functions. However, we additionally possess a discrete probability distribution $Q \in \Delta_n$, where $\Delta_n$ is the $n-$dimensional simplex. This can be thought of as the \textit{reference distribution}, with the guidance of which we would like to achieve local distributional robustness. Note that this is a generalization of Problem \eqref{robust-avg}, where the reference distribution is simply assumed to be uniform. With this addition, the problem becomes
%One may think of the reference distribution $Q$ as assigning an \textit{importance score} to each task. Indeed, in a practical scenario, one might obtain such a reference distribution $Q$, for instance by personally assigning an importance score to each task, or by externally observing the frequencies of each task being performed by some system. Thus, the reference distribution can facilitate incorporating the intent and priority of the main stakeholders in the multi-task subsection selection.

Based on this last formulation, we propose our robust formulation, which not only makes use of the reference distribution to simply weigh each objective function but to act as the center of a \textit{neighborhood of robustness}, within which we aim to be robust. The initial candidate formulation for this is:
\begin{equation}\label{novel-robust}
\begin{gathered}
\max_{\sS \subseteq \sN} \min_{P \in \Delta_n} \sum^n_{i=1} P_i f_i(\sS) \\
\text{s.t.}\; \lvert \sS\rvert \leq K, \; \gendiv{P}{Q} \leq R,
\end{gathered}
\end{equation}
where the constraint $\gendiv{P}{Q} \leq R$ designates the localization to a specified neighborhood of robustness within which we will maximize our objectives weighed by the worst possible distribution. $\gendiv{P}{Q}$ is intentionally unspecified and could be any statistical distance (although one natural choice quickly arises through analysis). For an investigation of different statistical distances for $\gendiv{P}{Q}$, we refer the reader to Section \ref{sec:distances}.

Before we proceed with a treatment of \eqref{novel-robust}, we propose one last change to make it more tractable. We work with a Lagrangian relaxation the formulation by introducing the constraint $\gendiv{P}{Q} \leq R$ into the objective via regularization constant $\lambda > 0$ rather than strictly enforcing it \cite{Boyd}. With this final addition, our formulation becomes
\begin{equation}\label{novel-final}
\begin{gathered}
\max_{\sS \subseteq \sN} \min_{P \in \Delta_n} \sum^n_{i=1} P_i f_i(\sS) + \lambda\gendiv{P}{Q}\\
\text{s.t.}\; \lvert \sS\rvert \leq K.
\end{gathered}
\end{equation}
% It must be remarked, at this point, that for appropriate choices\footnote{In particular, when relative entropy is used for $D$, as is the case in the better part of this paper. This choice will be made clear in Sections \ref{sec:distances} \& \ref{sec:theoretical}.} of the statistical distance $D$, the inner minimization problem of this last formulation is convex in $P$, and as such, under a very mild condition (namely, Slater's condition), the optimal solution for the regularized inner minimization objective in Problem \eqref{novel-final} will also be optimal for the hard-constrained previous formulation in Problem \eqref{novel-robust}\cite{Boyd}. Slater's condition requires the existence of a feasible point in the relative interior of the set delimited by the constraints. This condition, in the case of Problem (5), requires the existence of a point $P$ such that $\gendiv{P}{Q} < R$. The statistical distances examined in this paper will always satisfy this condition for any positive value of $R$, since the point $Q$ itself will constitute a feasible point in the relative interior. This, in turn, means that solving the  

We propose that this final formulation will be useful in many multi-objective applications where we aim to maintain a performance threshold across all tasks, with respect to a reference distribution indicating the relative importance of each task. In particular, we undertake two such applications in Section \ref{sec:results}. The first deals with a low Earth orbit satellite constellation performing multiple tasks involving atmospheric readings and ground coverage, where in critical missions it may be of importance to have satisfactory performance in all objectives. 
The second involves a task of image summarization, where each objective represents the cosine distance of one image in a given image dataset to the current summary chosen out of the dataset. It is clear that in such a task, one would expect a certain level of performance to be met with respect to all of the images, so that the resulting summary is representative of the whole dataset.
% where $\lambda \geq 0,$ and
% \begin{equation}
% \infdiv{P}{Q} = \sum_{i=1}^n P_i \log\lp\frac{P_i}{Q_i}\rp
% \end{equation}
% denotes the relative entropy of $P$ and $Q$, otherwise known as the Kullback-Leibler divergence operator \cite{elemsofinf}.
\subsection{Related Works and Significance}
\textbf{Robust submodular optimization.} 
A fundamental work in establishing the notion of robustness in submodular optimization is \cite{RSOS}, and the same notion is adopted in the current manuscript. The authors propose an algorithmic solution to the worst-case formulation of \eqref{robust-first}, namely, the \textsc{Submodular Saturation Algorithm (SSA)}. SSA will be discussed considerably in the following sections and will constitute one of the baselines we compare against. 
%Reference \cite{RSOS} is one of the seminal texts in robust submodular optimization literature, presenting the \textsc{Submodular Saturation Algorithm}, also discussed and used substantially in the current work for comparison purposes. This work proposes an algorithmic solution for Problem \eqref{robust-first}, aiming to optimize the worst-case task. 
In \cite{Powers2016ConstrainedRS, NEURIPS2018_7e448ed9}, the authors similarly adopt the same formulation of \eqref{robust-first} in consideration of robustness. The former generalizes the results of \cite{RSOS}, which deals with the cardinality-constrained problem, to cases where knapsack or other matroid constraints are present, including the case where the constraints themselves are submodular. The latter shifts the focus to achieving higher computational efficiency, aiming for a fast and practical algorithm with asymptotic approximation guarantees when the number of objective functions $n$ is increased arbitrarily. 
%References \cite{NEURIPS2018_7e448ed9, Powers2016ConstrainedRS} also focus on robustness in the sense of Problem \eqref{robust-first}. The first considers the problem of finding a fast and practical algorithm that has asymptotic approximation guarantees, in relation to a growing size of objective functions to be robust against. The second generalizes the result of \cite{RSOS}, to the cases of matroid, knapsack, and other submodular constraints.
In \cite{malherbe2022robustness}, the authors propose a novel notion of robustness, still in the context of multiple objective functions. They first motivate their novel approach by outlining the pessimistic nature of \eqref{robust-first}. They instead propose an approach that entails the maximization of the $p$th quantile of the objective functions. To simplify, this approach may be viewed as disregarding the worst-performing $p$th of the objectives in favor of maximizing the remaining. Different from all of the previously mentioned works, the present work shifts the focus to the scope of \textit{distributional robustness}, focusing on localizing our goal of robustness within a subset of $\Delta_n$, to a neighborhood of a reference distribution. 
%In \cite{malherbe2022robustness}, the authors formulate a novel approach to robustness, making use of quantiles. They draw attention to the pessimistic nature of focusing on maximizing the worst-case task and instead propose maximizing the ``$p$\textsuperscript{th} quantile'' of the multiple objective functions, in some sense ignoring the worst performing $p$\textsuperscript{th} fraction of the tasks and focusing on optimizing the rest. Different from these texts, we build our work on distributional robustness, focusing on robustness within a ``neighborhood'' of a reference distribution. 
In \cite{partitioning, removalrobust}, the authors adopt one of the other aforementioned notions of robustness. In these works, the notion of robustness is against the post-processing of solutions for the arbitrary or adversarial deletion of a certain number of elements, rather than robustness against multiple submodular objective functions.
%References \cite{partitioning, removalrobust} deal with an entirely different understanding of robustness in the context of submodular optimization. Rather than focusing on a multi-task scenario, they deal with the problem of robustness against post-processing of solutions constructed by optimization algorithms, which arbitrarily or adversarially remove elements of the solution set.

\textbf{Distributionally-robust optimization (DRO).} Our work is closely related to and motivated by the DRO paradigm. In a very general sense, DRO is a paradigm where some uncertainty about the nature of the problem is governed by a probability distribution and the goal is to leverage the information we have about this distribution with the goal of optimization or optimal decision-making. In \cite{distributional-convex}, the authors approach a convex optimization problem in the context of DRO. In \cite{dro-nonconvex}, the scope is broadened to the online optimization of nonconvex functions. In particular, the authors use relative-entropy regularization, which we also employ in this work. The authors of \cite{pmlr-v89-staib19a}, much like our work, deal with distributionally-robust optimization of submodular functions. They work in the presence of stochastic submodular objective functions that are drawn from a given reference distribution and aim to maximize their expected value. They propose an approach that entails a variance-regularized objective, making use of the multilinear extension of submodular functions and the Momentum Frank-Wolfe algorithm. In contrast to this line of work, we leverage DRO for discrete multi-task subset selection and demonstrate the possibility of formulating the distributionally-robust problem as a simple submodular maximization problem. This perspective essentially allows us to produce a distributionally-robust solution at no additional cost to producing a solution to \eqref{robust-weighted-avg}.

\section{Preliminaries and Background}
We begin by introducing the concept of the \textit{marginal return} in set functions.
\begin{definition}[Marginal return]
Given a set function $f\vcentcolon2^\sN\rightarrow\mathbb{R}$ and $\mathcal{S}, \mathcal{T} \subseteq \sN$, we denote $f(\mathcal{T} \cup \mathcal{S}) - f(\mathcal{T})$, the \textit{marginal return} in $f$ due to adding $\mathcal{S}$ to $\mathcal{T}$, by $\Delta_f(\sS\mid\mathcal{T})$. When $\sS$ is a singleton, i.e., $\sS = \{a\}$, we simply write $\Delta_f(a\mid\mathcal{T})$.
\end{definition} 
Now, we introduce the notion of submodularity using this definition.
\begin{definition}[Submodularity]\label{def:submod}
A set function $f\vcentcolon2^\sN\rightarrow\mathbb{R}$ is \textit{submodular} if for every $\sS\subseteq \mathcal{T} \subseteq \sN$ and $e \in \sN \setminus \mathcal{T},$ it holds that
\begin{equation}
\Delta_f(e\mid\sS) \geq \Delta_f(e\mid\mathcal{T}).
\end{equation}
\end{definition}
This definition of submodularity highlights the \textit{diminishing marginal returns} property. 
% However, the following equivalent
% definition, although less intuitive, proves to be more useful in certain situations, including a part of our analysis:
% \begin{definition}[Submodularity, alternative \cite{Narayanan1997SubmodularFA}]\label{def:alternative-submodularity}
% A set function $f\vcentcolon 2^N \rightarrow \mathbb{R}$ is submodular if for every $A, B \subseteq N,$
% \begin{equation}
% f(A\cap B) + f(A\cup B) \leq f(A) + f(B).
% \end{equation}
% \end{definition}
% An intuitive way of seeing the equivalence of the two definitions is to view $A \cup \{e\}$ and $ B$ in Definition \ref{def:submod} as $A$ and $B$, respectively, from Definition \ref{def:alternative-submodularity} \cite{lecture}.
The common notion in thinking about a submodular function $f$ is that it scores each set by its utility, as such, we will commonly refer to the value $f(\sS)$ as the \textit{utility}, \textit{performance}, or \textit{score} of set $\sS$.

We further introduce two additional useful properties of set functions, that are sought along submodularity in the derivation of theoretical guarantees for approximation algorithms.
\begin{definition}[Normalized set functions]
A set function $f\vcentcolon 2^\sN \rightarrow \mathbb{R}$ is \textit{normalized} if $f(\emptyset) = 0$.
\end{definition}
%\hfill\break
\begin{definition}[Monotone nondecreasing set functions]
A set function $f\vcentcolon 2^\sN \rightarrow \mathbb{R}$ is \textit{monotone nondecreasing} if for every $\sS\subseteq \mathcal{T} \subseteq \sN,$ we have $f(\sS) \leq f(\mathcal{T}).$
\end{definition}

Note that a set function that is both normalized and monotone nondecreasing is necessarily nonnegative.

In practical applications, it is often the case that the objective functions involved are not submodular, but \textit{weak submodular}. In literature, there are multiple relaxations of submodularity, but for the purposes of this paper, we adopt one that is well-established \cite{routing, approximate, quadratic, hibbard_hashemi_tanaka_topcu_2023, kaya2024randomizedgreedymethodsweak}, and reasonable in relaxing the diminishing marginal returns property. To facilitate the definition, we first define the \textit{weak-submodularity constant:}
\begin{definition}[Weak-submodularity constant (WSC) \cite{hibbard_hashemi_tanaka_topcu_2023}]\label{def:wsc} The WSC of a monotone nondecreasing set function $f$ is given by
\begin{equation}
w_f \vcentcolon= \max_{(\sS, \mathcal{T}, e) \in \Tilde{\sN}} \dfrac{\Delta_f(e\mid \mathcal{T})}{\Delta_f(e\mid \sS)},
\end{equation}
where $\Tilde{\sN} = \{(\sS, \mathcal{T}, e): \sS \subseteq \mathcal{T} \subset \sN, e \in \sN \setminus \mathcal{T}\}.$
\end{definition}
An intuitive understanding of the WSC is that it is the degree of the maximum violation of the diminishing marginal returns property of a set function. Evidently, a monotone nondecreasing set function $f$ is submodular if and only if its WSC satisfies $w_f \leq 1$.

We now define weak submodularity.
\begin{definition}[Weak submodularity]\label{def:weaksubmod}
A set function $f \vcentcolon 2^\sN \to \mathbb{R}$ is weak submodular if it is monotone nondecreasing and its WSC satisfies $w_f < \infty$.
\end{definition}

Possibly the most fundamental result in submodular optimization concerns the maximization of normalized, monotone nondecreasing submodular functions under a cardinality constraint. In this case, the iterative \textsc{Greedy} algorithm, which simply consists of going over the entire set of remaining elements at each iteration and selecting the one with the highest marginal return, enjoys an approximation guarantee of $1-1/e$. Although simple in principle, \textsc{Greedy} is demonstrably the optimal approximation algorithm, unless $\text{P}=\text{NP}$.\cite{nemhauser1978analysis} 

While \textsc{Greedy} is demonstrably optimal and easy to implement, its naive form is usually avoided due to the computational cost of reevaluating every remaining element at each iteration. Several methods have been proposed to circumvent this shortcoming, two significant ones being \textsc{Lazy Greedy}\cite{lazygreedy}, and \textsc{Stochastic Greedy}\cite{mirzasoleiman2015lazier}. In particular, \textsc{Stochastic Greedy} achieves a reduced computational cost by restricting its evaluations to a subset of the remaining elements sampled uniformly at random at each iteration. \textsc{Stochastic Greedy} also enjoys, in expectation, an approximation ratio similar to that of \textsc{Greedy}, with an additional term that designates the dependence on the cardinality of the randomly sampled subset.
%Although this is a celebrated result, the use of the standard \textsc{Greedy} algorithm is often avoided in practice, due to the computational cost of evaluating the marginal gain due to each remaining element in the ground set at each iteration of the algorithm. To circumvent this, a more efficient variant of the standard \textsc{Greedy} algorithm, oftentimes called \textsc{Stochastic Greedy} has been proposed. This algorithm achieves its efficiency by considering a randomly sampled subset of the remaining elements instead of its entirety at each iteration. It is also shown to achieve an approximation ratio similar to that of the standard \textsc{Greedy} algorithm, although with a dependence on the size of the sampled set used in each iteration. This result is formalized in the following theorem.
\begin{theorem}[\textsc{Stochastic Greedy} approximation ratio\label{thm:stoch-greedy}\cite{mirzasoleiman2015lazier}]
Let $f\vcentcolon 2^\sN \rightarrow \mathbb{R}$ be a normalized, monotone nondecreasing submodular function. Let $R = (\lvert \sN\rvert/K)\log(1/\epsilon)$ be the size of the sampled set at each iteration of \textsc{Stochastic Greedy} used in the solution of \eqref{cardinality}, where $\epsilon \in (0, 1)$. Then, in expectation, \textsc{Stochastic Greedy} achieves an approximation ratio of $1-1/e-\epsilon$, i.e., it produces a solution $\hat{\sS}$ which satisfies
\begin{equation}
    \E[f(\hat{\sS})] \geq (1 - 1/e - \epsilon) f(\sS^\ast),
\end{equation}
where $\sS^*$ is a solution of \eqref{cardinality}.
\end{theorem}
\section{Investigation of Various Statistical Distances for Regularization}\label{sec:distances}
In the formulation of \eqref{novel-final}, we have intentionally left unspecified the statistical distance $\gendiv{P}{Q}$ in the regularization term. In this section, we will be specifically interested in investigating which choices of $\gendiv{\cdot}{\cdot}$ lead to favorable algorithmic solutions and theoretical guarantees for Problem \eqref{novel-final}. In particular, statistical distances which preserve the submodular nature of the problem will be of specific interest. For, ensuring that the problem remains submodular allows us to produce a locally robust solution essentially at no additional cost with respect to the non-robust formulation of \eqref{cardinality}, through familiar tools such as \textsc{Stochastic Greedy}, and enjoying theoretical guarantees such as in Theorem \ref{thm:stoch-greedy}.

We first shift our attention to $L^p$ metrics. One choice for $\gendiv{P}{Q}$ that leads to a favorable algorithmic solution is the $L^\infty$ metric, i.e., $\gendiv{P}{Q} = \lVert P - Q \rVert_\infty \vcentcolon= \max_{i\in[n]} \lvert P_i - Q_i\rvert$. With this choice, \eqref{novel-final} becomes
\begin{equation}\label{eq:infty-norm}
\begin{gathered}
\max_{\sS \subseteq \sN} \min_{P \in \Delta_n} \sum^n_{i=1} P_i f_i(\sS) + \lambda\lVert P - Q \rVert_\infty\\
\text{s.t.}\; \lvert \sS\rvert \leq K.
\end{gathered}
\end{equation}
In this case, we can exploit the reduction of the inner maximization of \eqref{eq:infty-norm} to a linear program to derive a result. Since we know that the optimal point will occur at one of the vertices of the feasible region \cite{bertsimas-LPbook}, we have the following equivalence:
\begin{equation}\label{eq:equivalence}
\begin{split}
\min_{P \in \Delta_n} \sum^n_{i=1} P_i f_i(\sS) + \lambda\lVert P - Q \rVert_\infty
=
\min_{i \in [n]} f_i(\sS) - \lambda Q_i.
\end{split}
\end{equation}
Hence, \eqref{eq:infty-norm} reduces to
\begin{equation}\label{eq:saturate-w-pref}
\begin{gathered}
\max_{\sS \subseteq \sN} \min_{i \in [n]} f_i(\sS) - \lambda Q_i\\
\text{s.t.}\; \lvert \sS\rvert \leq K.
\end{gathered}
\end{equation}
The main observation about \eqref{eq:saturate-w-pref} is that it is almost identical the worst-case formulation of \eqref{robust-first}, however, in addition, the regularization by the reference distribution $Q$ still holds influence. This leads us to propose to use SSA with a slight modification enabling the incorporation of a reference distribution $Q$ in the solution of \eqref{eq:saturate-w-pref}. The full algorithm, which we name \textsc{Saturate with Preference}, is presented in Algorithm \ref{alg:ssa-w-pref}. It has a function evaluation complexity of $O(\lvert \sN\rvert^2 n \log(n \min_{i\in[n]} f_i (\sN))$, identical to the standard SSA \cite{RSOS}.

It is worth noting that a choice of the $L^1$ metric, i.e., $\gendiv{P}{Q} = \lVert P - Q \rVert_1$, through the same reduction of the inner problem to a linear program, and the same equivalence of \eqref{eq:equivalence}, results in the exact same formulation of \eqref{eq:saturate-w-pref}. Since $P$ and $Q$ are discrete probability distributions, one may also view the quantity $\lVert P - Q \rVert_1$ as a scaling of the total variation distance, an example of an $f$-divergence, which motivates the discussion of the next section, where we use one specific $f$-divergence with particularly significant theoretical outcomes, namely, the \textit{relative entropy} or \textit{KL-divergence} \cite{kullback1951information}.
\begin{algorithm}[t]
\caption{\textsc{Saturate with Preference}}
\label{alg:ssa-w-pref}
\hspace*{\algorithmicindent}\textbf{Input:} Finite family of monotone nondecreasing submodular functions $f_1, \ldots f_n$, ground set $\sN$, integer cardinality bound $0 \le K\le \lvert \sN \rvert$, regularization parameter $\lambda > 0$, reference distribution $Q \in \Delta_n$ \\
\hspace*{\algorithmicindent}\textbf{Output:} Solution set $\sS$
\begin{algorithmic}[1]
\STATE $k_m \leftarrow 0$
\STATE $k_M \leftarrow \min_{i\in[n]} f_i(\sN) - \lambda Q_i$
\STATE $\sS \leftarrow \emptyset$
\WHILE{$k_M - k_m \geq 1/n$}
\STATE $k \leftarrow (k_M - k_m)/2$
\STATE Define $\Bar{f}_{(k)}(\mathcal{X}) := \frac{1}{n}\sum_{i=1}^n \min\{f_i(\mathcal{X})-\lambda Q_i,\; k\}$
\STATE $\hat{\sS} \leftarrow \textsc{Greedy}(\Bar{f}_{(k)}, k)$
\IF{$\lvert\hat{\sS}\rvert > \alpha K$}
\STATE $k_M \leftarrow k$
\ELSE
\STATE $k_m \leftarrow k$
\STATE $\sS \leftarrow \hat{\sS}$
\ENDIF
\ENDWHILE
\STATE \textbf{return} $\sS$
%\RETURN: $\hat{x}\sim U\left[x_1, ...,x_{T+1}\right]$
\end{algorithmic}
\end{algorithm}
\section{Relative-Entropy Regularization}\label{sec:theoretical}
The choice of relative entropy for the regularizing function, i.e., $\gendiv{P}{Q} = \infdiv{P}{Q}$, where 
\begin{equation}
\infdiv{P}{Q} \vcentcolon= \sum_{i=1}^n P_i \log\lp\frac{P_i}{Q_i}\rp,
\end{equation}
leads to the most significant theoretical outcomes, as we will see in the following analysis. With this choice, the formulation becomes:
\begin{equation}\label{eq:novel-objective}
\begin{gathered}
\max_{\sS\subseteq\sN}\min_{P\in\Delta_n} \sum^n_{i=1} P_i f_i(\sS) + \lambda \infdiv{P}{Q} \\
\text{s.t.}|\sS| \leq K.
\end{gathered}
\end{equation}
We consider the dual of the inner minimization problem, by introducing a Lagrange multiplier for the constraint $P \in \Delta_n$. This constraint is equivalent to the two constraints \textbf{(i)} $P_i \geq 0$ for all $i \in [n]$ and \textbf{(ii)}  $\sum_{i=1}^n P_i = 1.$ Introducing a Lagrange multiplier only for constraint \textbf{(ii)}, we write the Lagrangian of the inner problem as
\begin{equation}
\mathcal{L}\lp P, \sS, \mu\rp = \sum_{i=1}^n P_i f_i(\sS) + \lambda \infdiv{P}{Q} + \mu\sum_{i=1}^nP_i - \mu.
\end{equation}
Minimizing $\mathcal{L}(P, \sS, \mu)$ over $P$ will yield the dual formulation of the inner problem of \eqref{eq:novel-objective}. We solve
\begin{equation}
\begin{gathered}
\min_{P\in \R^n} \mathcal{L}\lp P, \sS, \mu\rp \\
\text{s.t.}\; P_i \geq 0 \text{ for all }i \in [n].
\end{gathered}
\end{equation}
Leveraging the convexity of the problem in $P$, we simply look at the partial derivative of $\mathcal{L}$ with respect to each $P_i,$ which is given by
\begin{equation}\label{eq:optimalP}
\frac{\partial\mathcal{L}}{\partial P_i} = f_i(\sS) + \lambda \log \lp \frac{P_i}{Q_i} \rp + \lambda + \mu.
\end{equation}
Setting this equal to $0$, we obtain the minimizer $P_i^*$:
\begin{equation}\label{eq:pstar}
P_i^* = Q_i\exp{\lp-\frac{\lambda+\mu+f_i(\sS)}{\lambda}\rp}.
\end{equation}
Using this, along with the fact that $\sum_{i=1}^n P_i^* = 1,$ we get
\begin{equation}
\begin{split}
\sum_{i=1}^n P_i^* &= \sum_{i=1}^n Q_i\exp{\lp-\frac{\lambda+\mu+f_i(\sS)}{\lambda}\rp} = \sum_{i=1}^n Q_i\exp{\lp - \frac{f_i(\sS)}{\lambda}\rp}\exp{\lp-\frac{\lambda+\mu}{\lambda}\rp}=1.
\end{split}
\end{equation}
Then,
\begin{equation}\label{eq:lambda-plus-mu}
\lambda+\mu= \lambda\log\lp\sum_{i=1}^nQ_i\exp{\lp\frac{-f_i(\sS)}{\lambda}\rp}\rp.
\end{equation}
Now, evaluating $\mathcal{L}$ at the minimizer $P^*$, we obtain
\begin{equation}\label{eq:lagrangian-derivation}
\begin{gathered}
\mathcal{L}(P^*, \sS, \mu) = \sum_{i=1}^n P_i^* f_i(\sS) + \lambda \infdiv{P^*}{Q} \\
= \sum_{i=1}^n Q_i\exp{\lp-\frac{\lambda+\mu+f_i(\sS)}{\lambda}\rp}f_i(\sS) + \lambda \sum_{i=1}^n Q_i\exp{\lp-\frac{\lambda+\mu+f_i(\sS)}{\lambda}\rp} \cdot \lp-\frac{\lambda+\mu+f_i(\sS)}{\lambda}\rp \\
= -\lp\lambda+\mu \rp \sum_{i=1}^n Q_i\exp{\lp-\frac{\lambda+\mu+f_i(\sS)}{\lambda}\rp}. \\
\end{gathered}
\end{equation}
Since we have
\begin{equation}
\begin{split}
\sum_{i=1}^n Q_i\exp{\lp-\frac{\lambda+\mu+f_i(\sS)}{\lambda}\rp} = \sum_{i=1}^n P_i^*=1,
\end{split}
\end{equation}
\eqref{eq:lagrangian-derivation} becomes
\begin{equation}
\begin{split}
\mathcal{L}(P^*, \sS, \mu) = -\lp\lambda + \mu\rp.
\end{split}
\end{equation}
Finally, from \eqref{eq:lambda-plus-mu}, we obtain:
\begin{equation}
\begin{split}
G(\sS) \vcentcolon= \mathcal{L}(P^*, \sS, \mu)= -\lambda\log\lp\sum_{i=1}^nQ_i\exp{\lp\frac{-f_i(\sS)}{\lambda}\rp}\rp.
\end{split}
\end{equation}
Hence, the dual formulation of \eqref{eq:novel-objective} becomes that of the maximization of another set function $G$:
\begin{equation}\label{eq:novel-clean}
\begin{gathered}
\max_{\sS \subseteq \sN} G(\sS) \\%= -\lambda\log\lp\sum_{i=1}^n\exp{\lp\frac{-f^i(S)}{\lambda}\rp}Q_i\rp \\
\text{s.t.}\;\lvert \sS\rvert \leq K.
\end{gathered}
\end{equation}
A simple analysis of the extreme cases of  $\lambda$ provides good intuition on the equivalence of the two problems. Let us evaluate the limit
\begin{equation}
\begin{split}
\lim_{\lambda\to x} -\lambda\log\lp\sum_{i=1}^nQ_i\exp{\lp\frac{-f_i(\sS)}{\lambda}\rp}\rp = \lim_{\lambda\to x} - \frac{\log\lp\sum_{i=1}^nQ_i\exp{\lp\frac{-f_i(\sS)}{\lambda}\rp}\rp}{\frac{1}{\lambda}}.
\end{split}
\end{equation}
This limit is equivalent to
\begin{equation}
\lim_{\lambda\to x} \frac{\sum_{i=1}^nQ_i\exp{\lp\frac{-f_i(\sS)}{\lambda}\rp f_i(\sS)}}{\sum_{i=1}^nQ_i\exp{\lp\frac{-f_i(\sS)}{\lambda}\rp}}.
\end{equation}
Letting $x\to \infty$, we have
\begin{equation}\label{eq:recovered-1}
\begin{split}
\lim_{\lambda\to \infty} \frac{\sum_{i=1}^nQ_i\exp{\lp\frac{-f_i(\sS)}{\lambda}\rp f_i(\sS)}}{\sum_{i=1}^nQ_i\exp{\lp\frac{-f_i(\sS)}{\lambda}\rp}} = \sum_{i=1}^nQ_i f_i(\sS).
\end{split}
\end{equation}
This is indeed the expected behavior, since letting $\lambda\to \infty$ in \eqref{eq:novel-objective} effectively assigns all importance to the regularization term and forces one to have $P^\ast = Q$, making the objective
\begin{equation}
\begin{gathered}
\max_{\sS\subseteq \sN} \sum_{i=1}^nQ_i f_i(\sS)\\
\text{s.t.}\;\lvert \sS\rvert \leq K,
\end{gathered}
\end{equation}
which is recovered exactly in \eqref{eq:recovered-1}. Note further that $P^\ast=Q$ is attained only at the limit, for, combining \eqref{eq:pstar} and \eqref{eq:lambda-plus-mu}, we have an analytical expression for $P^\ast$, namely,
\begin{equation}\label{eq:optmalPi}
\begin{split}
&P_i^\ast = Q_i \cdot \exp\left(-\dfrac{\lambda\log\left(\sum_{j=1}^n Q_j \exp\left(\dfrac{-f_j(\sS)}{\lambda}\right)\right)+f_i(\sS)}{\lambda}\right).
\end{split}
\end{equation}
This means that for any finite value of $\lambda$, to get $P^\ast=Q$, we need, for all $i \in [n]$,
\begin{equation}
\exp\left(-\dfrac{\lambda\log\left(\sum_{j=1}^n Q_j \exp\left(\dfrac{-f_j(\sS)}{\lambda}\right)\right)+f_i(\sS)}{\lambda}\right)= 1,
\end{equation}
which, in turn, means
\begin{equation}
\lambda\log\left(\sum_{j=1}^nQ_j\exp\left(\dfrac{f_j(\sS)}{\lambda}\right)\right)=-f_i(\sS)
\end{equation}
for all $i \in [n]$. This last condition implies that $f_i(\sS) = f_j(\sS)$, for any two $i, j \in [n]$, i.e., for any finite value of $\lambda$,  the values of all objectives must have the same value at $\sS$ for $P^\ast = Q$. Outside of this very specific circumstance, at no finite value of $\lambda$ does $P^\ast=Q$ hold.

On the other hand, letting $x=0$, the value of the limit is determined by the most dominant terms in the sums in the numerator and the denominator, that is,
\begin{equation}
\begin{split}
&\lim_{\lambda\to 0} \frac{\sum_{i=1}^nQ_i\exp{\lp\frac{-f_i(\sS)}{\lambda}\rp f_i(\sS)}}{\sum_{i=1}^nQ_i\exp{\lp\frac{-f_i(\sS)}{\lambda}\rp}} = \min_{i\in[n]} \frac{Q_i\exp{\lp\frac{-f_i(\sS)}{\lambda}\rp f_i(\sS)}}{Q_i\exp{\lp\frac{-f_i(\sS)}{\lambda}\rp}} = \min_{i\in[n]} f_i(\sS).
\end{split}
\end{equation}
Hence, the problem becomes
\begin{equation}\label{eq:recovered-2}
\begin{gathered}
\max_{\sS\subseteq \sN} \min_{i\in[n]} f_i(\sS)\\
\text{s.t.}\;\lvert \sS\rvert \leq K.
\end{gathered}
\end{equation}
Again, letting $x=0$ in \eqref{eq:novel-objective} removes the regularization term, and reduces the inner minimization problem to a linear program, whose solution will have $P_i = 1$ for $i = \argmin_{i\in[n]} f_i(\sS)$, and $P_j = 0$ for all $j\neq i$, producing \eqref{eq:recovered-2}. Note that this is also exactly the worst-case formulation of \eqref{robust-first}, demonstrating that our formulation is a generalization of the worst-case formulation.

The main question that arises directly from the formulation of \eqref{eq:novel-clean} is whether the set function $G$ retains the properties of being normalized, monotone nondecreasing, and submodular since the presence of all three of these properties would reduce its solution to the use of standardized methods, such as \textsc{Stochastic Greedy}. Fortunately, we are not far from this ideal case, as demonstrated in Theorem \ref{prop:main}.
\begin{algorithm}[t]
\caption{\textsc{Stochastic Greedy}}
\label{alg:stoch-greedy}
\hspace*{\algorithmicindent}\textbf{Input:} Monotone nondecreasing submodular function $f$, ground set $\sN$, integer cardinality bound $ 0 \le K \le \lvert \sN \rvert$, random sampling set size $r$ \\
\hspace*{\algorithmicindent}\textbf{Output:} Solution set $\sS$
\begin{algorithmic}[1]
\STATE $\sS \leftarrow \emptyset$
\WHILE{$|\sS| < K$}
\STATE Sample a subset $\mathcal{R}$ of size $r$ uniformly at random from $\sN \setminus \sS$
\STATE $\sS \leftarrow \sS \cup \argmax_{e \in \mathcal{R}} \Delta_f(e\mid\sS)$
\ENDWHILE
\STATE \textbf{return} $\sS$
%\RETURN: $\hat{x}\sim U\left[x_1, ...,x_{T+1}\right]$
\end{algorithmic}
\end{algorithm}
\begin{theorem}\label{prop:main}
The set function $G(\sS)$ can be expressed as the composition of two functions, $G(S) = g(h(\sS))$, where the set function
\begin{equation}
h(\sS) \vcentcolon= \sum_{i=1}^n Q_i \lp1-\exp\lp-\frac{f_i(\sS)}{\lambda}\rp\rp
\end{equation}
is \textbf{(i)} normalized, \textbf{(ii)} monotone nondecreasing, and \textbf{(iii)} submodular, and the function $g(x) \vcentcolon= -\lambda\log(1-x)$ is \textbf{(iv)} monotone increasing, \textbf{(v)} convex, and \textbf{(vi)} Lipschitz continuous in $\operatorname{dom}g = [0, h(\sS^\ast)]$.
\end{theorem}
\begin{proof}
\begin{enumerate}[label=\bfseries(\roman*), leftmargin=*, itemsep=2pt]
\item We have
\begin{equation}
\begin{split}
h(\emptyset) = \sum_{i=1}^nQ_i\lp1-\exp{\lp-\frac{f_i(\emptyset)}{\lambda}\rp}\rp 
= \sum_{i=1}^nQ_i\lp1-\exp{\lp0\rp}\rp = 0.
\end{split}
\end{equation}
\item Let $\sS \subseteq \mathcal{T} \subseteq \sN$. Then, for each $i \in [n],$ we have $f_i(\sS) \leq f_i(\mathcal{T}).$ Hence, $-f_i(\sS)/\lambda \geq - f_i(\mathcal{T})/\lambda,$ and $\exp(-f_i(\sS)/\lambda) \geq \exp(-f_i(\mathcal{T})/\lambda),$ thanks to the monotone increasing property of the $\exp$ function and the positivity of $\lambda$. This means that 
\begin{equation}
1-\exp\lp -\frac{f_i(\sS)}{\lambda}\rp \leq  1-\exp\lp -\frac{f_i(\mathcal{T})}{\lambda}\rp,
\end{equation}
and because $Q_i \geq 0$ for all $i \in [n],$ 
\begin{equation}
\begin{split}
\sum_{i=1}^nQ_i\lp1-\exp\lp -\frac{f_i(\sS)}{\lambda}\rp \rp \leq \sum_{i=1}^nQ_i\lp1-\exp\lp -\frac{f_i(\mathcal{T})}{\lambda}\rp\rp,
\end{split}
\end{equation}
that is, $h(\sS) \leq h(\mathcal{T})$.
\item Let $\sS \subseteq \mathcal{T} \subset \sN$, $e \in \sN \setminus \mathcal{T}$. For all $i \in [n]$, we have
\begin{equation}
\dfrac{f_i(\sS \cup \{e\}) - f_i(\sS)}{\lambda} \geq \dfrac{f_i(\mathcal{T} \cup \{e\}) - f_i(\mathcal{T})}{\lambda}.
\end{equation}
Now, because $x\mapsto 1-\exp(-x)$ is a monotone increasing function, we have
\begin{equation}\label{eq:multiplythis}
\begin{split}
1-\exp\lp\dfrac{-f_i(\sS \cup \{e\}) + f_i(\sS)}{\lambda}\rp \geq 1-\exp\lp\dfrac{-f_i(\mathcal{T} \cup \{e\}) + f_i(\mathcal{T})}{\lambda} \rp.
\end{split}
\end{equation}
Furthermore, because each $f_i$ is monotone nondecreasing, $f_i(\sS) \leq f_i(\mathcal{T})$, and since $x\mapsto \exp(-x/\lambda)$ is a monotone decreasing function for $\lambda > 0$, we have
\begin{equation}
\exp\lp \frac{-f_i(\sS)}{\lambda} \rp \geq \exp\lp \frac{-f_i(\mathcal{T})}{\lambda} \rp.
\end{equation}
Both of these quantities are nonnegative. Then, multiplying the left-hand side of \eqref{eq:multiplythis} by $\exp(-f_i(\sS)/\lambda)$ and the right-hand side by $\exp(-f_i(\mathcal{T})/\lambda)$ preserves the inequality, and we obtain
\begin{equation}
\begin{split}
\exp\lp \frac{-f_i(\sS)}{\lambda} \rp - \exp\lp \frac{-f_i(\sS \cup \{e\})}{\lambda} \rp  \geq \exp\lp \frac{-f_i(\mathcal{T})}{\lambda} \rp - \exp\lp \frac{-f_i(\sT \cup \{e\})}{\lambda} \rp.
\end{split}
\end{equation}
This demonstrates that the set function $\sS \mapsto -\exp (-f_i(\sS)/\lambda)$ is submodular for all $i \in [n]$. The addition of the constant $1$, multiplication with the nonnegative $Q_i$ and the summation through the indices $i \in [n]$ preserve submodularity, hence,
\begin{equation}
h(\sS) = \sum_{i=1}^n Q_i \lp 1 - \exp\lp \dfrac{-f_i (\sS)}{\lambda}\rp\rp
\end{equation}
is submodular.
\item It suffices to note that the first-order derivative of $g$,
\begin{equation}
\dfrac{\mathrm{d}g}{\mathrm{d}x} = \dfrac{\lambda}{1-x} > 0,
\end{equation}
for all $\lambda >0$ and $x \in [0, h(\sS^\ast)]$. 
% Furthermore $\frac{\mathrm{d}g}{\mathrm{d}x}$ attains its maximum value at $x=h(S^\ast)$.
\item The second-order derivative of g,
\begin{equation}
\dfrac{\mathrm{d}^2g}{\mathrm{d}x^2} = \dfrac{\lambda}{(1-x)^2} > 0,
\end{equation}
for all $\lambda >0$ and $x \in [0, h(\sS^\ast)]$, hence $g$ is a convex function.
\item As per \textbf{(v)}, the first-order derivative of $g$ is a monotone increasing function over its domain, attaining its maximum value at $x = h(\sS^\ast)$. Hence, 
\begin{equation}
\left\lvert\dfrac{\mathrm{d}g}{\mathrm{d}x}\right\rvert \leq \dfrac{\lambda}{1-h(\sS^\ast)},
\end{equation}
making $g$ Lipschitz continuous with constant $\lambda/(1-h(\sS^\ast))$.
\end{enumerate}
\end{proof}
This theorem proves that standard methods such as \textsc{Stochastic Greedy} (presented in Algorithm \ref{alg:stoch-greedy}), with well-established theoretical guarantees such as that of Theorem \ref{thm:stoch-greedy} are suitable for the solution of \eqref{eq:novel-clean} (and hence for its equivalent, \ref{eq:novel-objective}). This is due to the fact that $g$ is a monotone increasing function, and that the maximizer of $g$ over its domain is $h(\sS^\ast)$. Hence, any method approximating $h(\sS^\ast)$ may be employed in maximizing $G$ as well. Regarding the approximation guarantee provided by \textsc{Stochastic Greedy} for the function $G(\sS) = g(h(\sS))$, we state the following theorem.
\begin{theorem} The set $\sS$ constructed by
\textsc{Stochastic Greedy} in the solution of \eqref{eq:novel-clean} satisfies
\begin{equation}
\E [ G(\sS)] \geq G(\sS^\ast) - (1/e + \epsilon) \dfrac{h(\sS^\ast)}{1-h(\sS^\ast)}.
\end{equation}
\end{theorem}
\begin{proof}
We have, by Theorem \ref{thm:stoch-greedy},
\begin{equation}
\E[h(\sS)] \geq (1-1/e-\epsilon) h(\sS^\ast).
\end{equation}
This implies
\begin{equation}
\begin{split}
h(\sS^\ast) - \E[h(\sS)] &\leq h(\sS^\ast) - (1-1/e-\epsilon) h(\sS^\ast) = (1/e+\epsilon) h(\sS^\ast).
\end{split}
\end{equation}
Now, by the Lipschitz continuity of $g$ with constant $\lambda/(1-h(\sS^\ast))$ (as per Theorem \ref{prop:main}, \textbf{(vi)}), we have
\begin{equation}\label{eq:first}
g( h(\sS^\ast)) - g(\E[h(\sS)]) \leq ( h(\sS^\ast) - \E[h(\sS)]) \dfrac{\lambda}{1-h(\sS^\ast)}.
\end{equation}
Also, because $g$ is a convex function (as per Theorem \ref{prop:main}, \textbf{(v)}), by Jensen's inequality,
\begin{equation}\label{eq:second}
g\lp h(\sS^\ast)\rp - \E[g(h(\sS))] \leq g( h(\sS^\ast)) - g(\E[h(\sS)]).
\end{equation}
Then, combining \eqref{eq:first} and \eqref{eq:second} and rearranging, we have
\begin{equation}
\E [ G(\sS)] \geq G(\sS^\ast) - (1/e + \epsilon) \dfrac{h(\sS^\ast)}{1-h(\sS^\ast)}.
\end{equation}
\end{proof}
This, in turn, guarantees the production of a locally distributionally-robust solution that is produced at no additional function evaluation cost with respect to the solution of the naive formulation of \eqref{robust-weighted-avg}. Our novel formulation greatly relaxes the SSA and enables the use of much less computationally expensive methods for the solution. Namely, \textsc{Stochastic Greedy} has a function evaluation cost of $O(K\cdot\lvert \mathcal{R}\rvert),$ compared to SSA's $O(\lvert \sN\rvert^2 n \log(n \min_{i\in[n]} f_i (\sN)).$

Furthermore, this result extends gracefully to the case where the component objective functions $f_i$ for $i \in [n]$ are not submodular, but weak submodular instead. In this case, akin to the result with submodular functions, the newly-defined function $g$ in \eqref{eq:novel-clean} turns out to be weak submodular. The next lemma demonstrates this result.
\begin{lemma}
$G(\sS) = -\lambda\log(\sum_{i=1}^nQ_i\exp{(\frac{-f_i(\sS)}{\lambda})})$ is weak submodular when the component functions $f_i$ are weak submodular for all $i \in [n]$.
\end{lemma}
\begin{proof}
As per Definition \ref{def:wsc}, we have that
\begin{equation}
\begin{split}
w_G &= \max_{(\sS, \sT, e) \in \Tilde{\sN}} \dfrac{\Delta_G(e\mid \sT)}{\Delta_G(e\mid \sS)}=\max_{(\sS, \sT, e) \in \Tilde{\sN}} \dfrac{\log \lp \dfrac{\sum_{j=1}^n Q_j \exp\lp\dfrac{-f_j(\sT)}{\lambda}\rp}{\sum_{j=1}^n Q_j \exp\lp\dfrac{-f_j(\sT\cup \{e\})}{\lambda}\rp}\rp}{\log \lp \dfrac{\sum_{j=1}^n Q_j \exp\lp\dfrac{-f_j(\sS)}{\lambda}\rp}{\sum_{j=1}^n Q_j \exp\lp\dfrac{-f_j(\sS\cup \{e\})}{\lambda}\rp}\rp}.
\end{split}
\end{equation}
Clearly, this value can only be unbounded if the denominator is zero, and this only occurs when
\begin{equation}
\dfrac{\sum_{j=1}^n Q_j \exp\lp\dfrac{-f_j(\sS)}{\lambda}\rp}{\sum_{j=1}^n Q_j \exp\lp\dfrac{-f_j(\sS\cup \{e\})}{\lambda}\rp} = 1,
\end{equation}
that is, when
\begin{equation}
\sum_{j=1}^n Q_j \exp\lp\dfrac{-f_j(\sS)}{\lambda}\rp = \sum_{j=1}^n Q_j \exp\lp\dfrac{-f_j(\sS\cup \{e\})}{\lambda}\rp.
\end{equation}
Rearranging the terms, this is equivalent to
\begin{equation}\label{eq:impossible}
\sum_{j=1}^n Q_j \underbrace{\lp \exp \lp \dfrac{-f_j(\sS)}{\lambda}\rp - \exp \lp \dfrac{-f_j(\sS \cup \{e\})}{\lambda} \rp \rp}_{A_j} = 0.
\end{equation}
Note that because each $f_j$ is monotone nondecreasing, the $A_j$ term in the sum is nonnegative for all $j$. Now, because $\sum_{j=1}^n Q_j = 1$, there exists at least one $j$ such that $Q_j >0$. Let $j_0$ be one such $j$. We know that $A_{j_0}$ is either positive or zero. If it is the case that $A_{j_0}$ is positive, then there is at least one positive term in the sum (alongside other nonnegative terms), and hence the sum cannot be equal to zero. If, on the other hand, $A_{j_0}$ is zero, then we must have that
\begin{equation}
f_{j_0}(\sS) = f_{j_0}(\sS \cup \{e\}).
\end{equation}
Yet, if this were the case, the WSC of $f_{j_0}$ would be
\begin{equation}
\begin{split}
w_{f_{j_0}} &= \max_{(\sS, \sT, e) \in \Tilde{\sN}} \dfrac{\Delta_{f_{j_0}}(e \mid \sT)}{\Delta_{f_{j_0}}(e\mid \sS)} = \dfrac{\Delta_{f_{j_0}}(e \mid \sT)}{f_{j_0}(S \cup \{e\}) - f_{j_0}(\sS)} = \dfrac{\Delta_{f_{j_0}}(e \mid \sT)}{0}= \infty,
\end{split}
\end{equation}
regardless of the choice of $\sT$ in the numerator. However, by the initial supposition that $f_j$ is weak-submodular for each $j \in [n]$, we know that $f_{j_0} < \infty$, and hence $A_{j_0}$ cannot be zero. Then, \eqref{eq:impossible} cannot hold, and thus $w_G < \infty$, establishing that $G$ is a weak-submodular function.
\end{proof}
This last result makes available to us theoretical guarantees related to \textsc{Stochastic Greedy} on weak submodular functions\cite{hibbard_hashemi_tanaka_topcu_2023, kaya2024randomizedgreedymethodsweak}.

With all of this settled, one point that remains to be examined is whether solving the regularized, ``soft-constrained'' problem \eqref{novel-final}, with which the entirety of this work has been interested, translates to solving the ``hard-constrained'' original problem \eqref{novel-robust}. We first note, as we previously have, the similarity of the formulation of \eqref{novel-final} to the Lagrangian of the inner minimization objective in \eqref{novel-robust}. Indeed, if we were to write the Lagrangian of the latter, we would have
\begin{equation}\label{eq:trueLagrange}
\Tilde{\mathcal{L}}(P, \sS, \lambda) = \sum^n_{i=1} P_i f_i(\sS) + \lambda\infdiv{P}{Q} -\lambda  R.
\end{equation}
Hence, the main difference between this Lagrangian and \eqref{novel-final} is the term $-\lambda R$. In conventional fashion, in finding the dual of this problem, we would first minimize the Lagrangian with respect to the primal variable $P$ in terms of the dual variable $\lambda$, and obtain
\begin{equation}
\Tilde{G}(\sS, \lambda) \vcentcolon= \sum^n_{i=1} P^\ast_i(\lambda) f_i(\sS) + \lambda\infdiv{P^\ast(\lambda)}{Q} -\lambda  R.
\end{equation}
Fortunately, the derivation starting in \eqref{eq:optimalP} remains valid, and so does the expression for $P^\ast$ given in \eqref{eq:optmalPi}, even with the addition of the $-\lambda R$ term. Then, continuing in the expected fashion, we would maximize $\Tilde{G}(\sS)$ with respect to $\lambda$ to obtain $\lambda^\ast$, which we would use in the expression for $P^\ast$, to find the optimal value for the primal variable, $P^\ast(\lambda^\ast)$, this time free of the variable $\lambda$. 

Let us for a moment adopt this approach in solving the inner minimization objective in \eqref{novel-final}. In maximizing
\begin{equation}
\begin{split}
\sum^n_{i=1} P^\ast_i(\lambda) f_i(\sS) + \lambda\infdiv{P^\ast(\lambda)}{Q}
 = -\lambda\log\lp\sum_{i=1}^nQ_i\exp{\lp\frac{-f_i(\sS)}{\lambda}\rp}\rp
\end{split}
\end{equation}
with respect to $\lambda$, we quickly realize that for $\lambda \geq 0$, the function is monotone increasing over its domain, and hence the maximizing $\lambda^\ast$ is unbounded. This, in turn, as the analysis in \eqref{eq:recovered-1} shows, forces $P^\ast = Q$.

This is precisely due to the missing $-\lambda R$ term. The inclusion of this term in the true Lagrangian $\Tilde{\mathcal{L}}$ in \eqref{eq:trueLagrange} is the key to having a finite optimal $\lambda^\ast$, and hence having $P^\ast \neq Q$. The inclusion of this term makes both $\lambda^\ast$ and $P^\ast$ dependent on $R$, the radius of the neighborhood of robustness:
\begin{equation}\label{eq:isconcave}
\lambda^\ast = \argmax_{\lambda\geq0} \;-\lambda\log\lp\sum_{i=1}^nQ_i\exp{\lp\frac{-f_i(\sS)}{\lambda}\rp}\rp -\lambda R.
\end{equation}
% and
% \begin{equation}
% \begin{split}
% &P_i^\ast(R) = Q_i \cdot \\ &\exp\left(-\dfrac{\lambda^\ast(R)\log\left(\sum_{j=1}^n Q_j \exp\left(\dfrac{-f^j(S)}{\lambda^\ast(R)}\right)\right)+f^i(S)}{\lambda^\ast(R)}\right).
% \end{split}
% \end{equation}
Because \eqref{eq:isconcave} is the maximization of a concave function, taking the derivative of the function with respect to $\lambda$ and setting it equal to $0$, we obtain the relation between $\lambda^\ast$ and $R$:
\begin{equation}\label{eq:therelation}
\begin{split}
&\dfrac{\sum_{i=1}^n Q_i \exp \lp\dfrac{-f_i(\sS)}{\lambda^\ast}\rp f_i(\sS)}{\lambda^\ast\sum_{i=1}^n Q_i \exp \lp\dfrac{-f_i(\sS)}{\lambda^\ast}\rp} + \log\lp \sum_{i=1}^n Q_i \exp \lp\dfrac{-f_i(\sS)}{\lambda^\ast} \rp\rp + R = 0.
\end{split}
\end{equation}
With an appeal to the implicit function theorem, we can (informally) say that within a neighborhood of any $(\lambda^\ast_0, R_0)$ pair that satisfies this relation, one can express $R$ as a one-to-one (hence invertible) function of $\lambda^\ast$, say, $R = r(\lambda^\ast)$.
% \begin{equation}
% \begin{split}
% \log\lp\sum_{i=1}^n Q_i \exp\lp\dfrac{-f^i(S)}{\lambda}\rp\rp + \dfrac{\sum_{i=1}^n Q_i\exp\lp\frac{-f^i(S)}{\lambda}\lp f^i(S)}{\lambda\sum_{i=1}^n Q_i\exp\lp\frac{-f^i(S)}{\lambda}\lp}\rp + R = 0
% \end{split}
% \end{equation}

The conclusion to these observations is the following: We claim that setting the regularization hyperparameter $\lambda$ to a value, say to $\lambda_0$, in \eqref{novel-final} is similar to setting the $R$ in \eqref{novel-robust} such that $R = r(\lambda_0)$. In other words, one can envision changing the value of the regularization parameter $\lambda$ in the ``soft-constrained'' problem as equivalent to changing the value of the radius of the neighborhood of robustness in the ``hard-constrained'' problem. This is intuitively supported by the observation that the $\lambda^\ast$ values obtained through the relation \eqref{eq:therelation} as $R\to 0^+$ approach $\infty$, and those obtained as $R\to\infty$ approach $0$.
\section{Application to Online Submodular Optimization}\label{sec:online}
A natural application of the proposed scheme arises within the context of \textit{online} submodular optimization, in the presence of time-varying objective functions. In this setting, we aim to produce a sequence of solutions to a sequence of problems
\begin{equation}
\begin{gathered}
\max_{\sS_t \in \sN} f_t(\sS_t) \\
\text{s.t.}\; \lvert \sS_t \rvert \leq K, \; \forall t \in \mathbb{N},
\end{gathered}
\end{equation}
where each $f_t$ is submodular, monotone nondecreasing, and normalized.

It is easy to approach this problem with the standard methods that we possess. We could treat the problem at each time step $t$ as a separate, standalone problem, and use any variant of the \textsc{Greedy} methods to produce a solution $\sS_t,$ entirely decoupled from the problems arising at different time steps.

However, in practical settings, it may not be desirable to reconstruct a new solution independent of the previous solutions at each time step, especially if the selection of additional distinct elements incurs extra cost. In this case, it would be desirable to construct a single solution that would perform satisfactorily over multiple time steps, if possible. One strategy to achieve this may be to set an observation window of $t_w$ time steps, so that over $t_w$ time steps we observe the objective functions \textit{played} by the system, and only after $t_w$ time steps will we play our solution suitable for the $t_w$ observations. Ideally, if there is some structure or notion of continuity within the variation of the objective functions $f_t,$ employing such a strategy might be worthwhile, as it would use up to $t_w$ times as few distinct elements while still achieving satisfactory results. In some sense, this is like an inverse of \textit{fairness} considerations, where one customarily wants to diversify their selections. Rather, here, the aim is to limit diversification and conserve the same selection of elements over many time steps.

With this idea in mind, we have to consider what strategy to use to leverage the $t_w$ observed objective functions $f_1, \ldots,f_{t_w}$ over a given observation window. One approach is to refer to a common idea in other optimization areas (e.g., momentum-based stochastic gradient descent methods in continuous optimization \cite{kingma2014adam,das2022faster} or the TD$(\lambda)$ method in reinforcement learning \cite{sutton2018reinforcement}). We initially propose the following scheme: For some fixed $\gamma \in [0, 1]$, we consider
\begin{equation}
\begin{split}
\max_{\sS \in \sN} \gamma^{t_w} f_1(\sS) + (1-\gamma) \sum_{t=1}^{t_w} \gamma^{(t_w-t)} f_t(\sS) 
\end{split}
\end{equation}
What this formulation aims to achieve is to geometrically weigh each of the $t_w$ observed objective functions, where $\gamma$ decides the relative importance of older observations with respect to the newer observations. A value of $\gamma = 1$ assigns all the weight to $f_1$, whereas $\gamma=0$ assigns, with the convention that $0^0 = 1,$ all the weight to the most recently observed objective function $f_{t_w}.$ The values $\gamma \in (0, 1)$ provide us with the whole spectrum of relative importances.

One important remark is that the geometric weighing of the objective functions over time effectively gives us a reference distribution $\Gamma$, which we can instead use as the center of a neighborhood of robustness, resulting in the formulation
\begin{equation}\label{eq:tr}
\begin{gathered}
\max_{\sS\subseteq \sN} \min_{P \in \Delta_{t_w}} \sum_{t=1}^{t_w} P_t f_t(\sS) + \lambda \infdiv{P_t}{\Gamma} \\
\text{s.t.}\;\lvert \sS \rvert \leq K, 
\end{gathered}
\end{equation}
where $\Gamma = ((1-\gamma)\gamma^{t_w-1},  \ldots, (1-\gamma)\gamma, (1-\gamma))$. This formulation aims to then carry the concept of robustness into the online setting, where we attempt distributional robustness over time steps of the problem. 
\section{Experimental Results}\label{sec:results}
\begin{figure*}%
    \centering
    \subfloat[\centering Performance on reference distribution]{{\includegraphics[width=.49\linewidth]{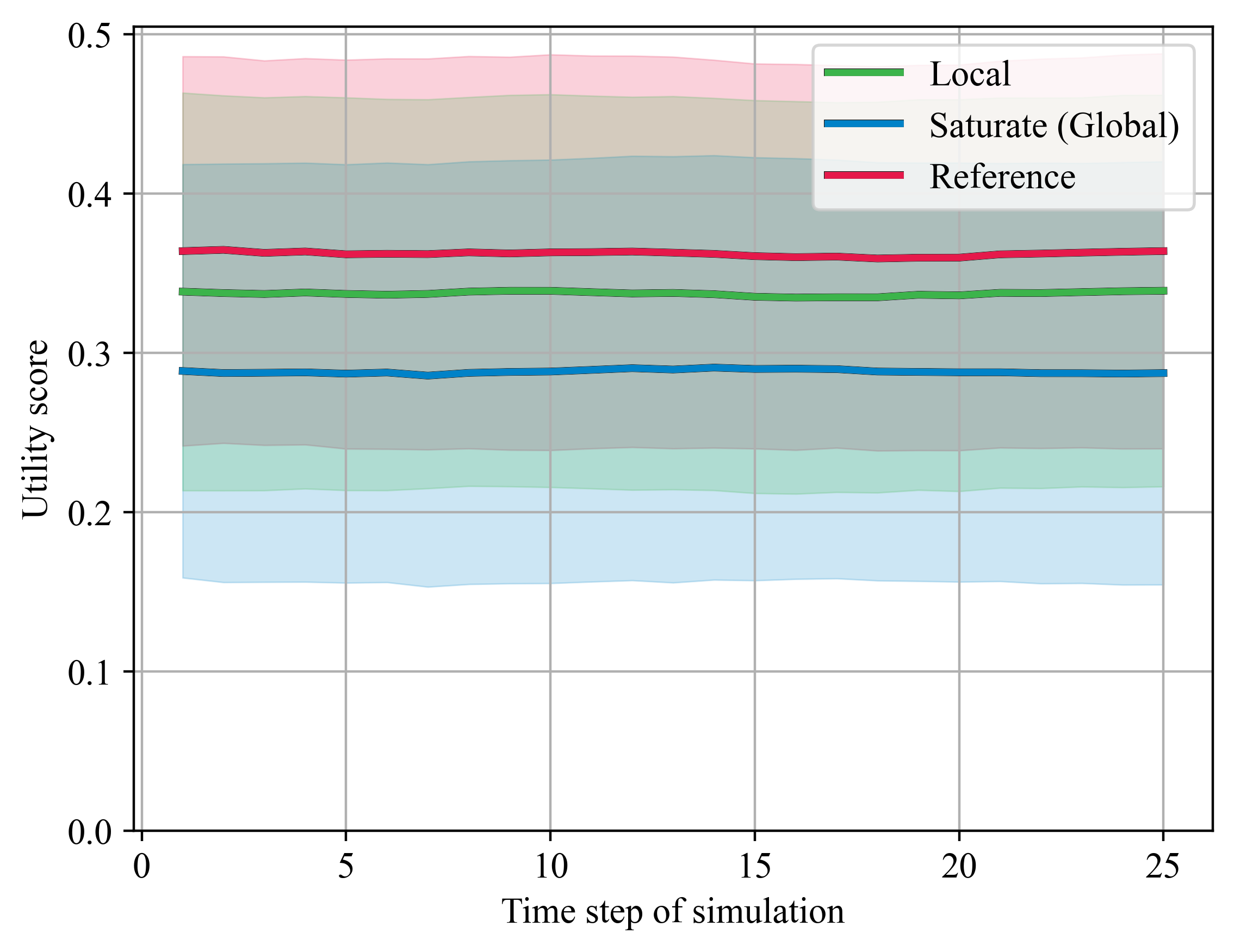} }}%
    %\quad
    \subfloat[\centering Performance on worst-case task]{{\includegraphics[width=.49\linewidth]{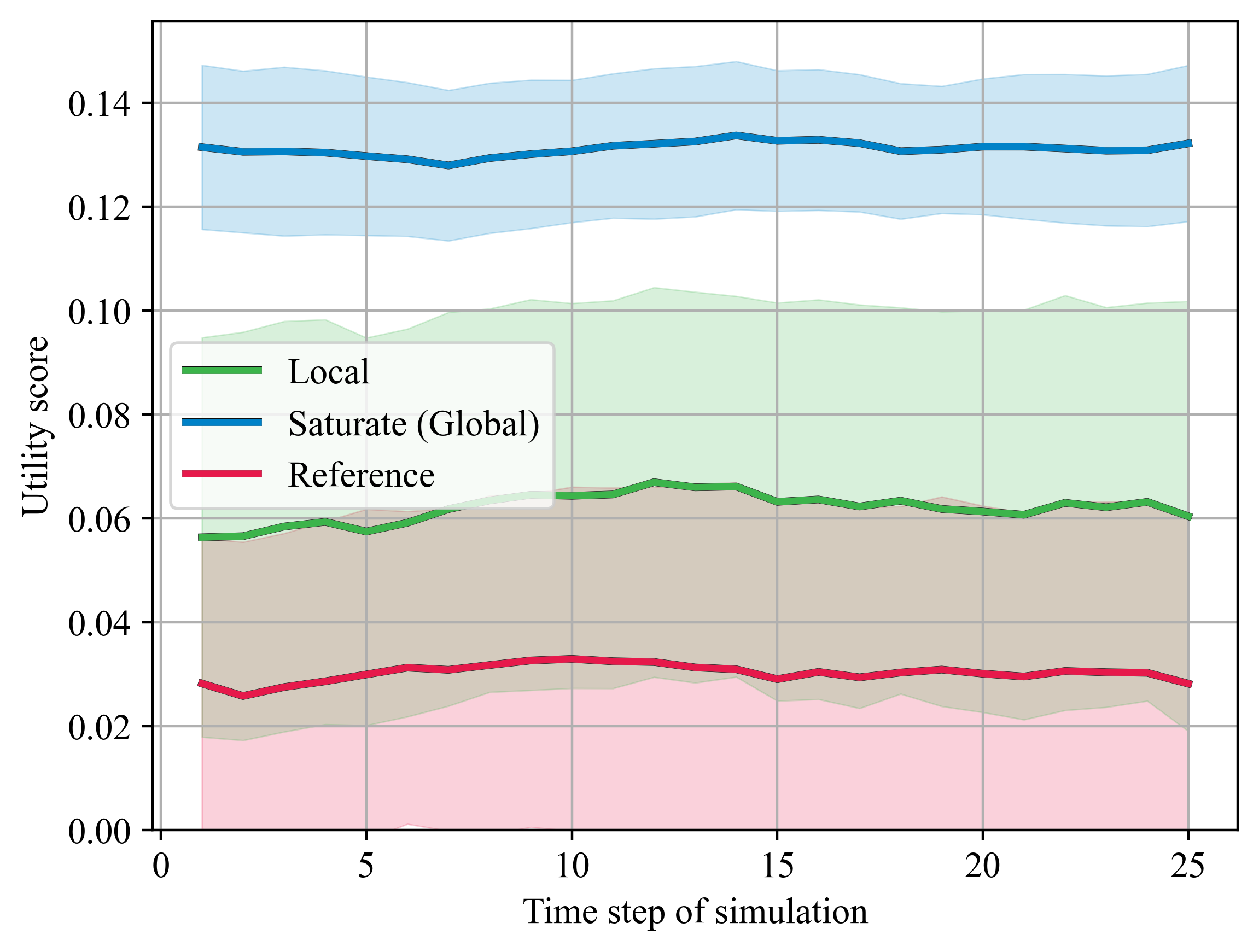} }}%
    \quad
    \subfloat[\centering Performance on local worst-case distribution]{{\includegraphics[width=.49\linewidth]{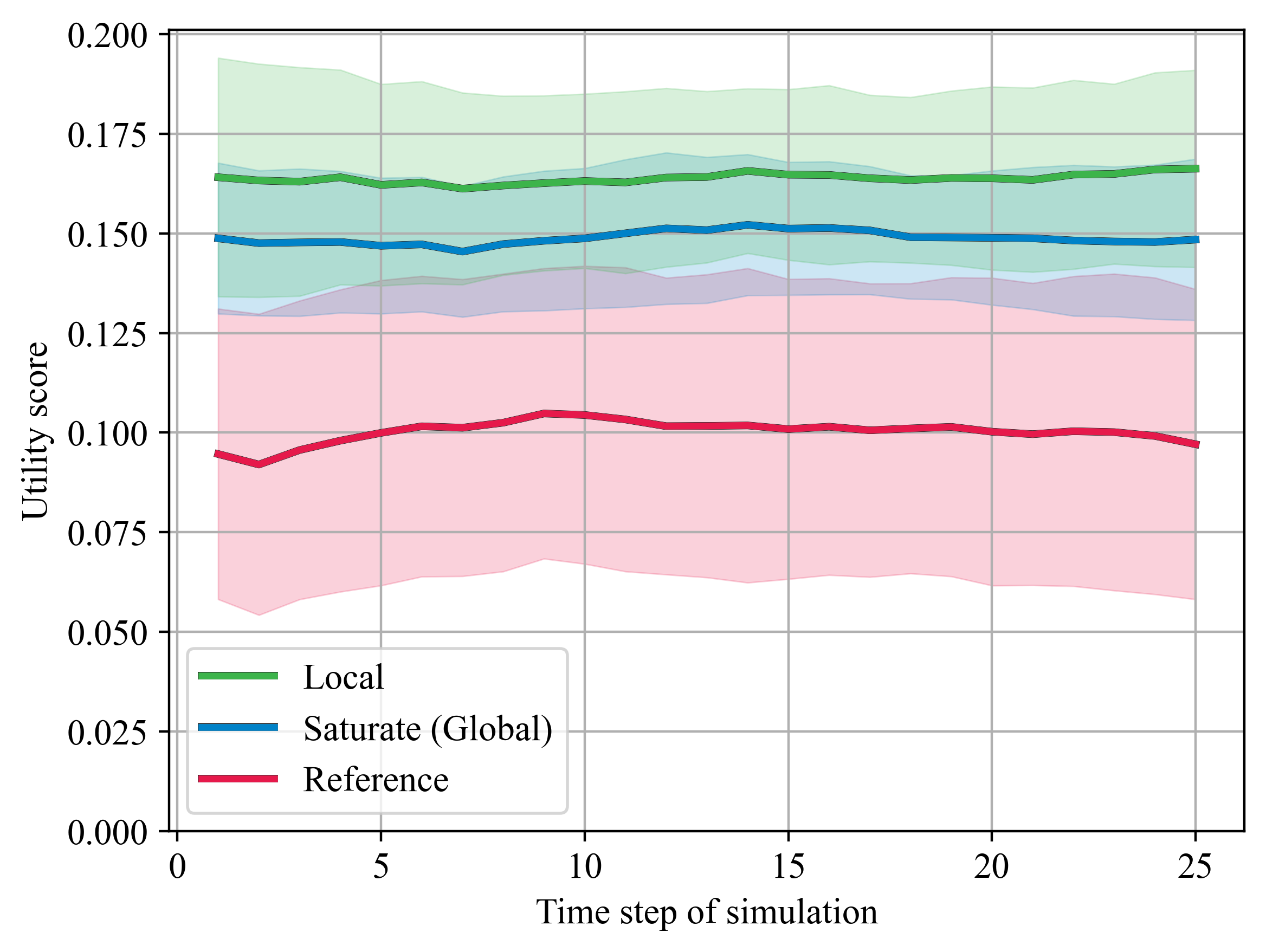} }}%
    %\quad
    \subfloat[\centering Wall-clock time taken]{{\includegraphics[width=.49\linewidth]{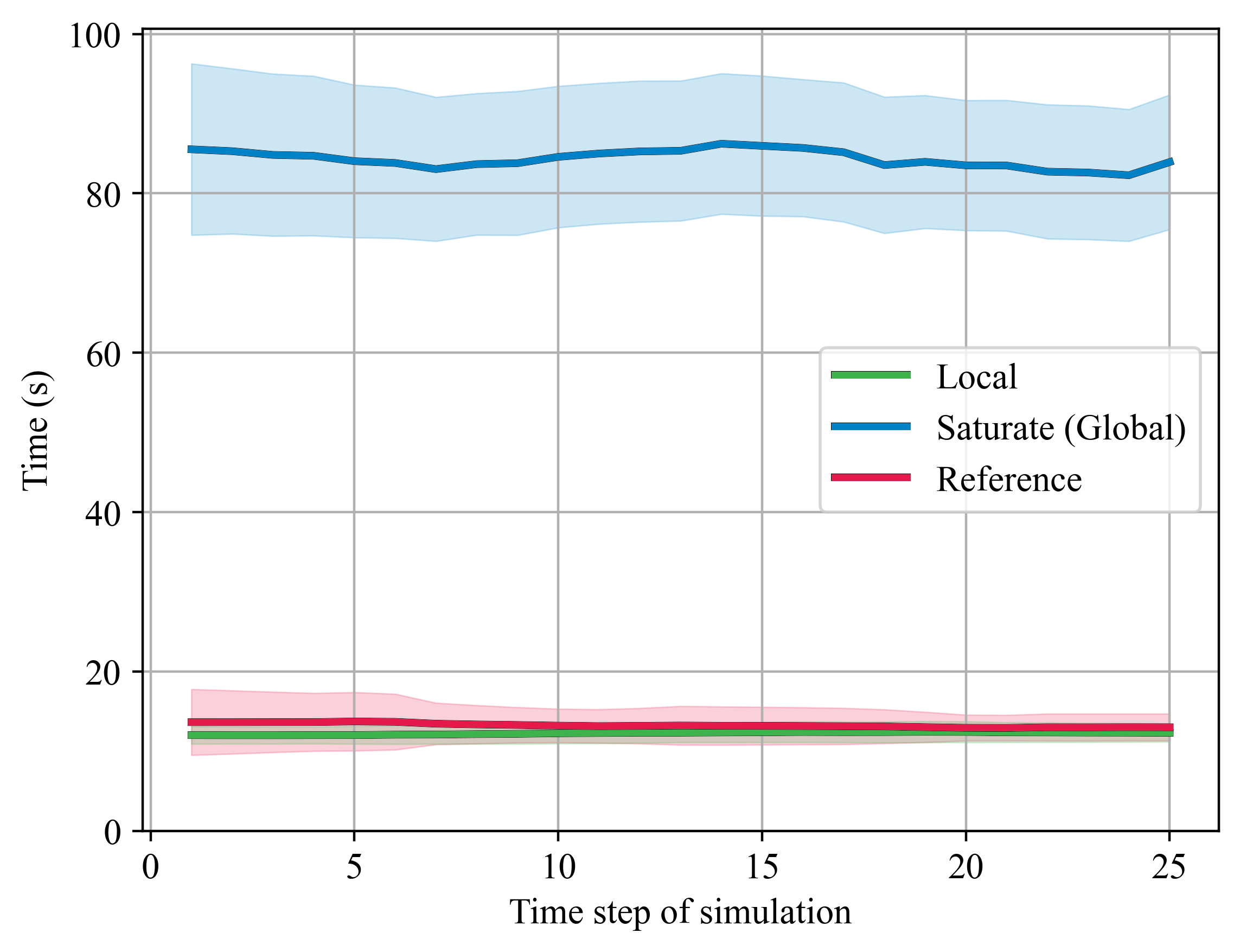} }}%
    \caption{The average performances over $15$ runs of the three algorithms focused on optimizing the reference distribution, the global worst-case task, and the local worst-case distribution close to the reference distribution, evaluated on the four criteria of reference distribution performance, worst-case task performance, local worst-case distribution performance and the wall-clock time taken by the algorithm in the construction of the solution in the satellite selection task of Section \ref{subsec:satsel}. \textbf{Local} represents the relative-entropy-regularized \textsc{Stochastic Greedy} algorithm solving our novel formulation, aiming for \textit{local} worst-case distributional robustness in the neighborhood of the reference distribution. \textbf{Saturate (Global)} represents SSA proposed in \cite{RSOS}, aiming for global worst-case task robustness. \textbf{Reference} represents the \textsc{Stochastic Greedy} algorithm being used to directly optimize the utility of the reference distribution. The highlighted areas indicate one standard deviation. The results have been put through a moving average filter with window size $6$.}%
    \label{fig:main}%
\end{figure*}
In this section, we use an application involving a constellation of low Earth orbit (LEO) satellites to test the performance of our three novelly proposed algorithms, namely, the relative-entropy-regularized \textsc{Stochastic Greedy} algorithm detailed in Section \ref{sec:theoretical}, its application to the online submodular optimization setting detailed in Section \ref{sec:online}, and the \textsc{Saturate with Preference} algorithm detailed in Section \ref{sec:distances}. For the results, we simulate a Walker-Delta constellation parameterized by $i:T/P/f$, where $i$ is the orbit inclination, $T$ is the number of satellites in the constellation, $\Pi$ is the total number of orbital planes of the constellation, and $\phi$ is the phase difference in between the orbital planes in pattern units\cite{hibbard_hashemi_tanaka_topcu_2023, constellations}. The semi-major axis length of all of the orbits is $8378.1$ kilometers, and the satellites remain Earth-pointing with a field-of-view angle of $\pi/3$ radians. For the testing of all algorithms, we instantiate a constellation with parameters $75^\circ:240/12/1$.

%We now direct our attention to the practical corroboration of our theoretical results through numerical experiments. We are motivated in this section by an application of our results to a sensor selection problem involving a constellation of low Earth orbit (LEO) satellites. Our setting involves the simulation of a Walker-Delta constellation parameterized by $i:T/P/f$, where $i$ is the orbit inclination, $T$ is the number of satellites in the constellation, $P$ is the total number of orbital planes of the constellation, and $f$ is the phase difference in between the orbital planes in pattern units.\cite{hibbard_hashemi_tanaka_topcu_2023, constellations} The semi-major axis length of all of the orbits is $8378.1$ kilometers, which corresponds roughly to a constellation altitude of $2000$ kilometers. The satellites remain Earth-pointing throughout the simulation with a field-of-view angle of $\pi/3$ radians.

In the testing of the relative-entropy-regularized \textsc{Stochastic Greedy} and the \textsc{Saturate with Preference} algorithms, we instantiate six distinct, time-varying tasks that yield six objective functions $f_{t,1}, \ldots, f_{t,6}$. The objective functions $f_{t,1}, \ldots, f_{t,5}$ are related to the performance of the algorithms on atmospheric sensing tasks. Each of these tasks involves taking atmospheric readings at a set of five randomly located points on Earth. The atmospheric conditions at these points are described by the Lorenz 63 model\cite{DeterministicNonperiodicFlow}. The dynamics at these points are parameterized by values that make the system chaotic. An illustration of one instance of these points is provided in Figure \ref{fig:selection-points}, where points labeled $1$ belong to atmospheric reading task $1$ whose performance is indicated by the value of $f_{t,1}$, and so on. In particular, the utility values $f_{t,1}(\sS), \ldots, f_{t,5}(\sS)$ of these tasks for a subset of selected satellites $\sS$ are proportional to the additive inverse of the mean-squared error (MSE) achieved by the selection $\sS$ of satellites. The MSE is calculated by estimating the state of each atmospheric sensing point using an unscented Kalman filter \cite{BarShalom2001EstimationWA}, and comparing it with the actual values of the atmospheric state at the points of interest. This estimation is highly dependent on whether an atmospheric point is within the field of view of a satellite. Note also that the additive inverse of the MSE is known to be a weak-submodular function\cite{leveraging, nearoptimal}. Specifically, for $i \in [5]$, $f_{t,i}$ has form
\begin{equation}
f_{t,i}(\sS_t) = \frac{1}{Z_{t, i}}\operatorname{Tr}\lp\mathbf{P}_{t|t-1} - \mathbf{F}_{\sS_t}^{-1}\rp,
\end{equation}
where $\mathbf{P}_{t|t-1}$ indicates the prediction error covariance matrix at time $t$, $\mathbf{F}_{\sS_t}^{-1}$ indicates the inverse of the Fisher information matrix for the solution $\sS_t$ constructed at time $t$\cite{nearoptimal}, and $Z_{t,i}$ is the normalizing coefficient of objective $i$ at time $t$. 
\begin{figure}
\includegraphics[width=1\linewidth]{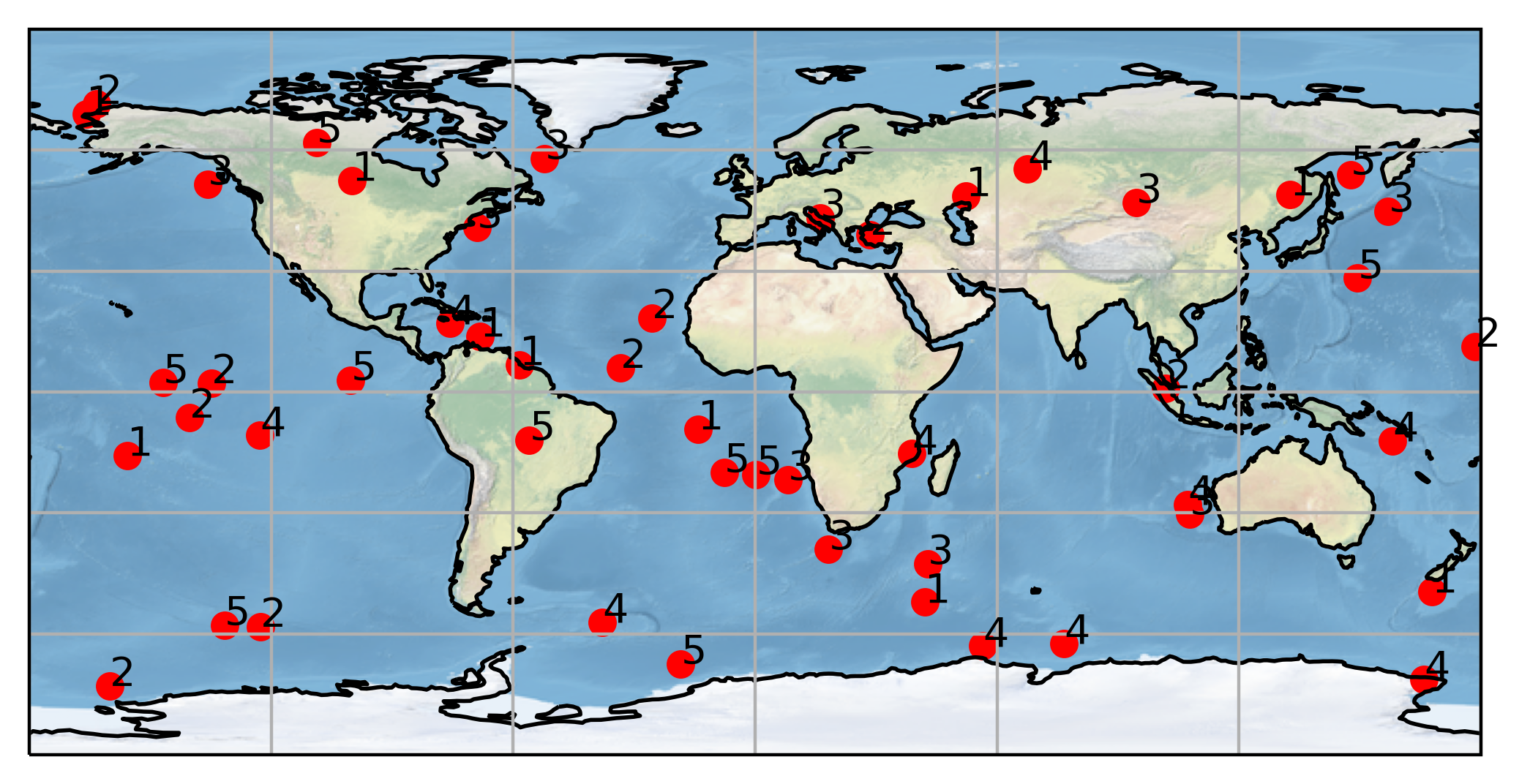}
  \caption{A selection of atmospheric points of interest for the five atmospheric reading tasks, $f_{t,1}, \ldots, f_{t,5}$ in one run of the simulation. Each task instantiates $5$ points of interest, for a total of $25$ points. The labels near the points indicate which atmospheric task a point belongs to.}
  \label{fig:selection-points}
\end{figure}
The sixth and final task, whose performance is given by the objective function $f_{t,6}$, involves ground coverage of the Earth. The utility value $f_{t,6}(\sS)$ of this task is proportional to the Earth coverage achieved by the selected subset $\sS$ of the satellites in the constellation. The area of coverage is determined based on a grid of the Earth's surface consisting of cells of width and height $2^\circ$. The area of each cell is calculated assuming a spherical Earth. If the center point of a cell on the grid is within the field-of-view of any satellite in the selection, that grid is taken to be covered by the selection. Specifically, $f_{t,6}$ has form
\begin{equation}
f_{t,6}(\sS_t) = \dfrac{1}{Z_{t,6}} \dfrac{\sum_{c=1}^C \mathbb{1}(\text{cell } c \text{ is in view of a satellite in } \sS_t)}{C},
\end{equation}
where $C$ is the total number of grid-cells. 

The utility values of all six tasks $f_{t,1},\ldots f_{t,6}$ are normalized to the range of $[0, 1]$, by dividing the utility $f_{ti}(\sS)$ of a selection $\sS$ at any time step by $t$ $f_{t,i}(\sN)$, the maximum utility achievable by selecting the entire ground set at that time step. This ensures that the importance score of an individual task is not influenced by the potentially arbitrary value of its utility, and is solely determined by the weight $Q_i$ assigned to it.

We randomly sample a reference distribution $Q$ uniformly from $\Delta_n$ to assign an importance score to each task. Following, we simulate the atmospheric states of the points of interest and the trajectory of the Walker-Delta constellation, for $25$ time steps, each corresponding to a time interval of $60$ seconds, for $15$ runs. 

\subsection{Relative-Entropy-Regularized Stochastic Greedy for Satellite Selection}\label{subsec:satsel}
For the assessment of the relative-entropy-regularized \textsc{Stochastic Greedy} algorithm, we compare the performance of the solutions produced by three algorithms, in terms of four distinct performance criteria. For each of the algorithms, we use a cardinality bound $K=10$, a sampling set size $\lvert\mathcal{R}\rvert=24$, and a regularization parameter $\lambda=0.1$, where applicable. The three algorithms compared are as follows:
\begin{itemize}
    \item \textbf{Algorithm 1 - Local:} The relative-entropy-regularized \textsc{Stochastic Greedy} algorithm, that is, the \textsc{Stochastic Greedy} algorithm applied to the maximization of our novel objective function $G(\sS) = -\lambda\log(\sum_{i=1}^nQ_i\exp{(-f_i(\sS)/\lambda)})$ in  \eqref{eq:novel-clean}.
    \item \textbf{Algorithm 2 - Saturate (Global):} The \textsc{Submodular Saturation Algorithm}, which aims to optimize the worst-case scenario of \eqref{robust-first}.
    \item \textbf{Algorithm 3 - Reference:} The \textsc{Stochastic Greedy} algorithm, applied to the maximization of the objective function $f(\sS) =\sum_{i=1}^nQ_if_i(\sS)$ of \eqref{robust-weighted-avg}, focused on directly optimizing with respect to the reference distribution with no consideration of local robustness.
\end{itemize}
The four performance criteria are as follows:
\begin{itemize}
    \item \textbf{Criterion 1:} The utility scores with respect to the objective $F_1(\sS) = \sum_{i=1}^n Q_i f_i(\sS)$, which represents the performance of the algorithms when weighed directly by the reference distribution $Q$.
    \item \textbf{Criterion 2:} The utility scores with respect to the objective $F_2(\sS) = \min_{i\in[n]} f_i(\sS)$, which represents the global worst-case single-task performance.
    \item \textbf{Criterion 3:} The utility scores with respect to the objective $F_3(\sS) = \sum_{i=1}^n P^*_i f_i(\sS)$, where $P^* = \argmin_{P\in\Delta_n} \sum^n_{i=1} P_i f_i(\sS) + \lambda \infdiv{P}{Q}$, hence representing the local worst-case scenario performance localized to a relative-entropy neighborhood of the reference distribution. This criterion may be considered as the benchmark of our novel formulation.
    \item \textbf{Criterion 4:} The wall-clock runtime of the algorithm.
\end{itemize}
The average performances of the three algorithms over $15$ runs of the simulation as evaluated by these four criteria are presented in Figure \ref{fig:main}.

Several observations follow from these results. Firstly, the designation of SSA as ``too pessimistic'' is indeed justified by the results, as seen from Figure \ref{fig:main} (a), where it fails to perform on the reference distribution as it is too focused on worst-case performance. However, it does indeed dominate in worst-case single-task performance over the other two algorithms, as evidenced by Figure \ref{fig:main} (b).

\begin{figure}[t]
    \centering
    \subfloat[\centering Selection made by Algorithm 1 - Local]{{\includegraphics[width=.49\linewidth]{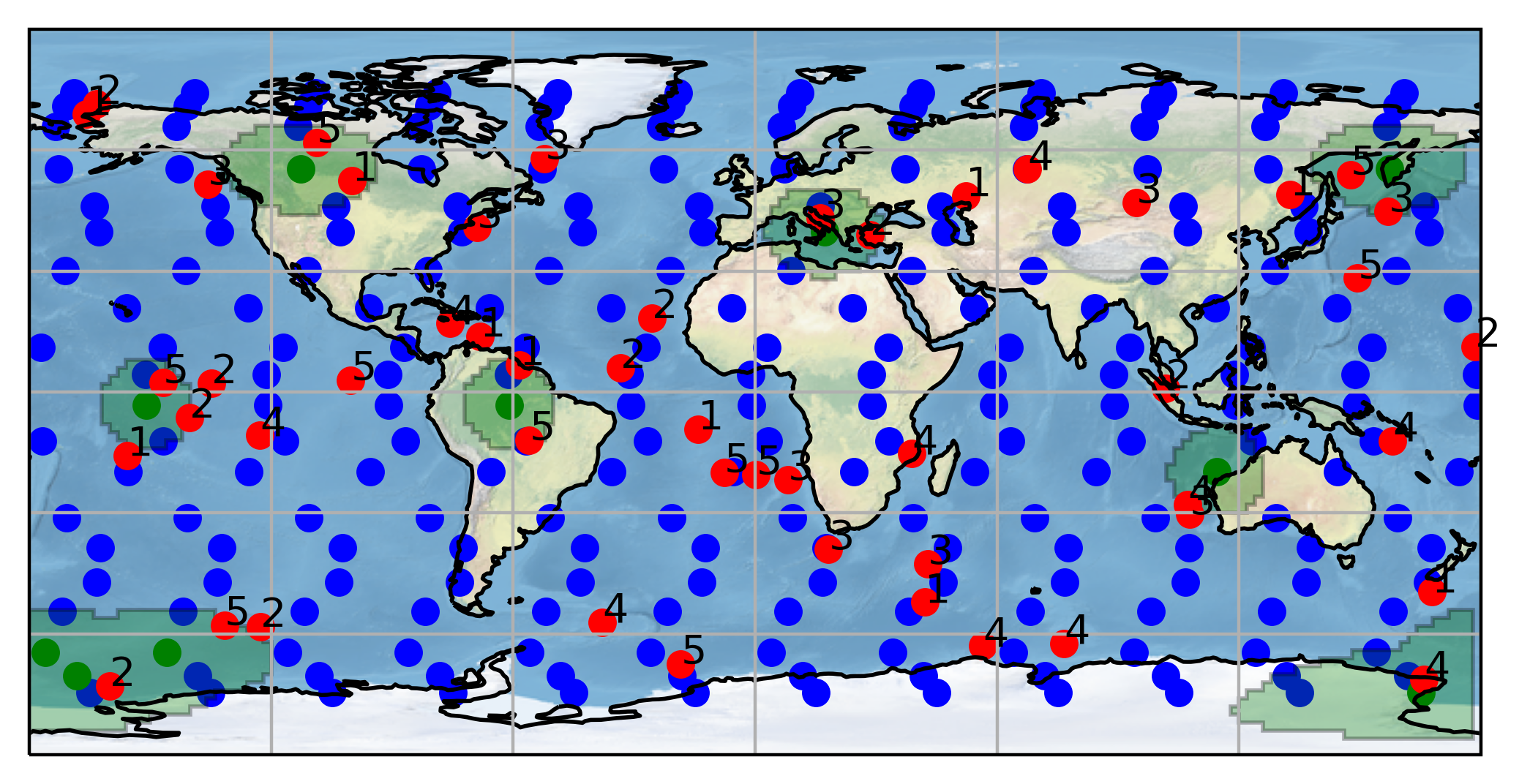} }}%
    \subfloat[\centering Selection made by Algorithm 2 - Saturate (Global)]{{\includegraphics[width=.49\linewidth]{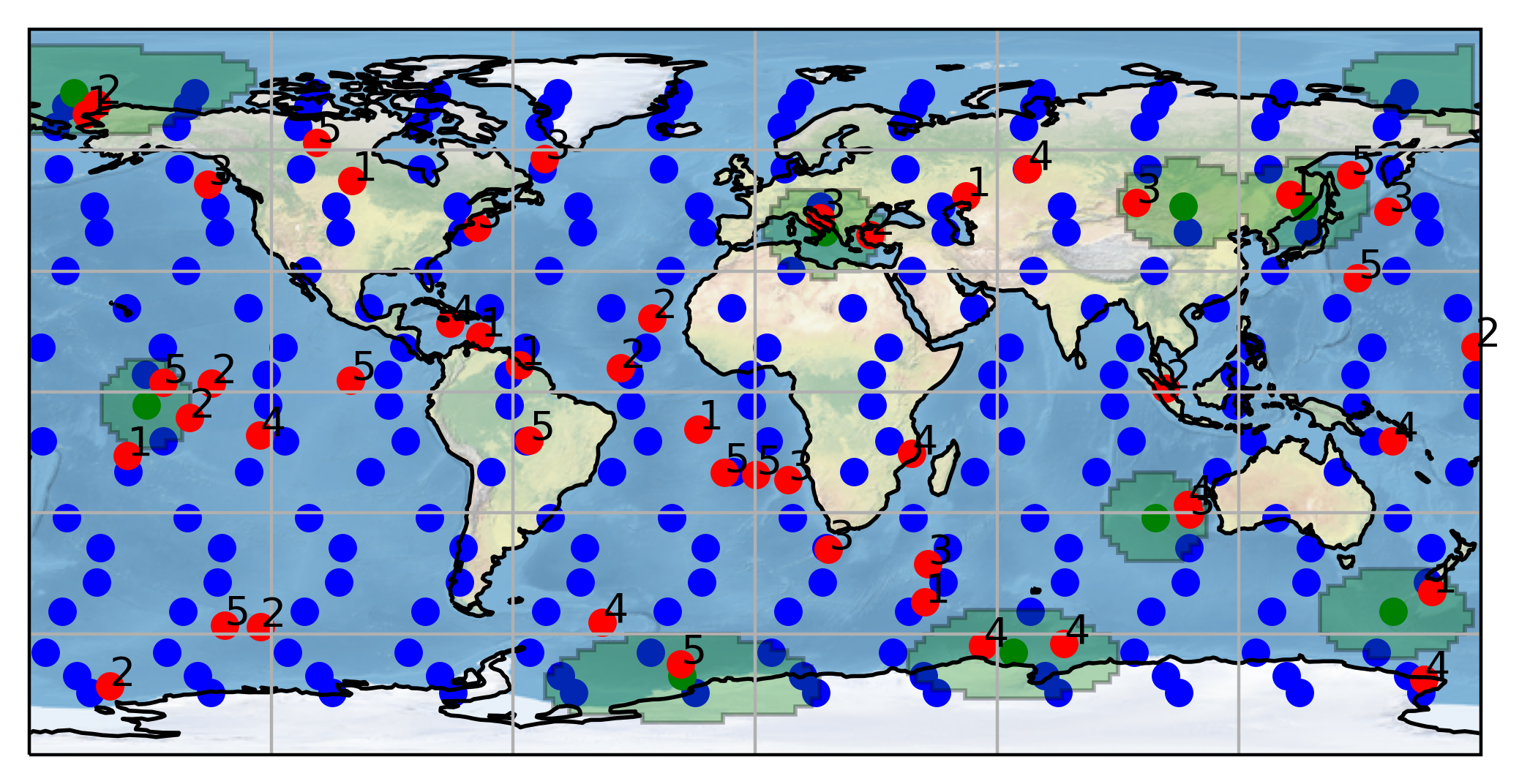}}}%
    \quad
    \subfloat[\centering Selection made by Algorithm 3 - Reference]{{\includegraphics[width=.49\linewidth]{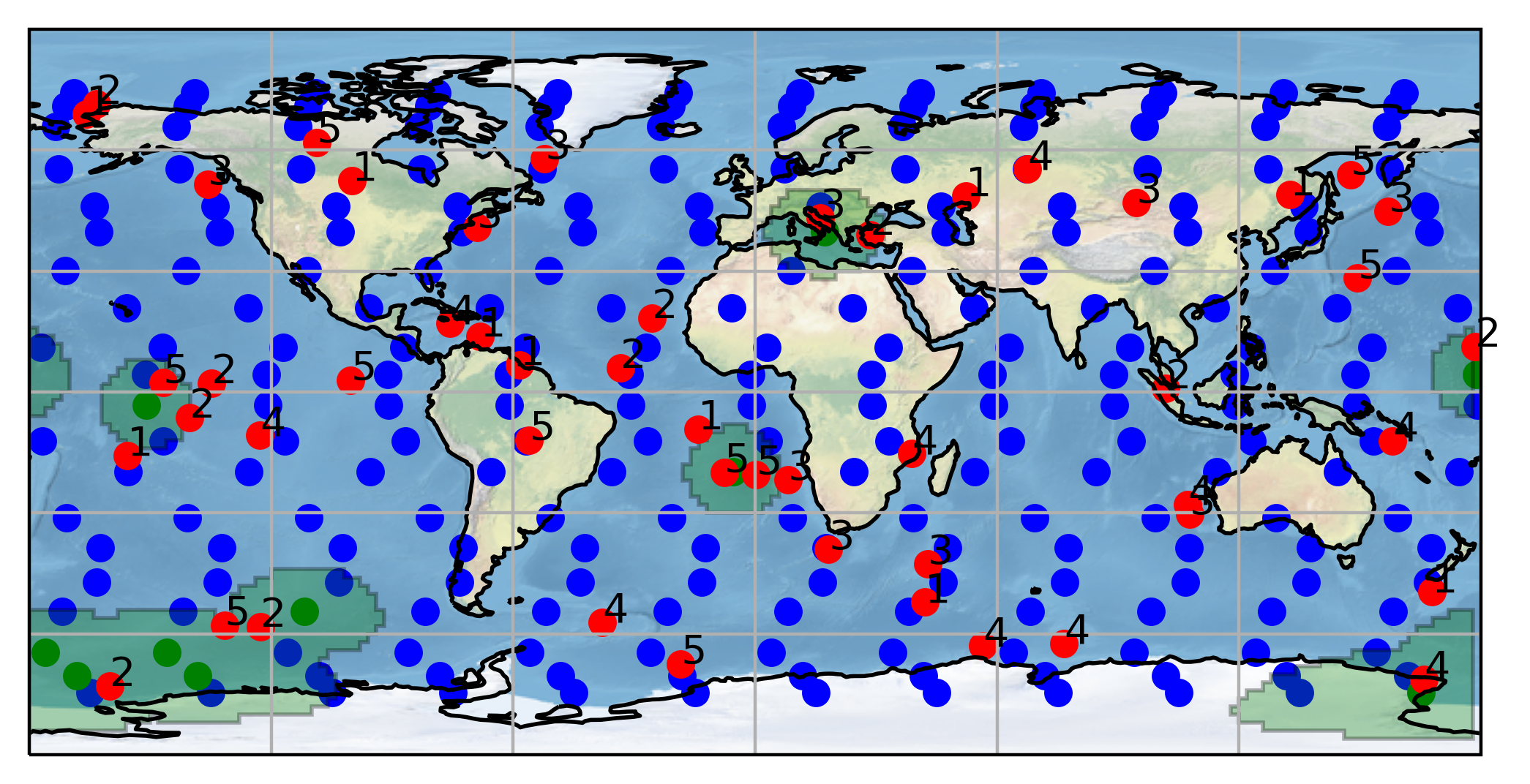}}}%
    \caption{The satellites selected by Algorithms 1, 2, and 3 on the $25$th time iteration of one run of the simulation. The red points with their corresponding numbers indicate the atmospheric points of interest and the tasks they belong to. The blue points indicate the satellites in the constellation. The green points indicate the selected satellites at the current time iteration. The highlighted green areas indicate the ground coverage provided by the selected satellites. The reference distribution for this run of the simulation is $Q = [0.022, 0.267, 0.088, 0.087, 0.183, 0.353].$}%
    \label{fig:frames}%
\end{figure}

Concerning the performance of Algorithm 1 - Local solving our formulation, we see from Figure \ref{fig:main} (a) that it has comparable performance on the reference distribution with Algorithm 3 - Reference, which directly optimizes the reference distribution. The usefulness of the Local approach is reflected clearly in Figure \ref{fig:main} (c), where we observe it fulfilling the local distributional robustness that our formulation aimed at achieving. Algorithm 1 - Local greatly outperforms Algorithm 3 - Reference and marginally outperforms SSA on the worst-case scenario within a neighborhood around the reference distribution. However, this marginal surpassing is complemented by the much smaller runtime and computational complexity of the Local approach.

With regard to the wall-clock time taken by the algorithms, as expected, Algorithms 1 and 3 perform identically, as they are instances of the same \textsc{Stochastic Greedy} algorithm. SSA performs much more poorly, consistently taking many times as much time at each time step. This is explained by the fact that due to its line search-based approach, SSA uses several runs of the \textsc{Stochastic Greedy} procedure in each of its iterations.

We may conclude from these results that Algorithm 1 - Local, solving our proposed novel formulation, succeeds in constructing a solution that is comparable to Algorithm 3 - Reference in optimizing the performance of the reference distribution, while also achieving local distributional robustness within the neighborhood of the reference distribution. Although SSA also manages to produce a solution that is more robust against worst-case tasks and achieves decent local distributional robustness, it is much more computationally expensive and has a much longer runtime.

\subsection{Saturate with Preference for Satellite Selection}\label{subsec:swp}
We now turn to the assessment of Algorithm \ref{alg:ssa-w-pref}, \textsc{Saturate with Preference}. The general setting of the experiments remains similar to the previous subsection, with the same constellation parameters of $75^\circ:240/12/1$, and the same objective functions $f_{t,1}, \ldots, f_{t,6}$. We compare the performance of the selection made by \textsc{Saturate with Preference}, with that of the unmodified SSA, by looking at the values achieved by the two objective functions that are assigned the highest weight by the random sampling of $Q$. In essence, this allows us to evaluate whether adding the element of preference to SSA works as intended. Indeed, Figure \ref{fig:swp} demonstrates that for the two objective functions with the highest priority, \textsc{Saturate with Preference} leads to the selection of a subset that consistently achieves a higher utility score in comparison to SSA. Indeed, Figure \ref{fig:swp} demonstrates that \textsc{Saturate with Preference} consistently outperforms SSA in terms of the performance on the objective functions with the highest assigned priority.
\begin{figure}[t]
    \centering
    \includegraphics[width=.5\linewidth]{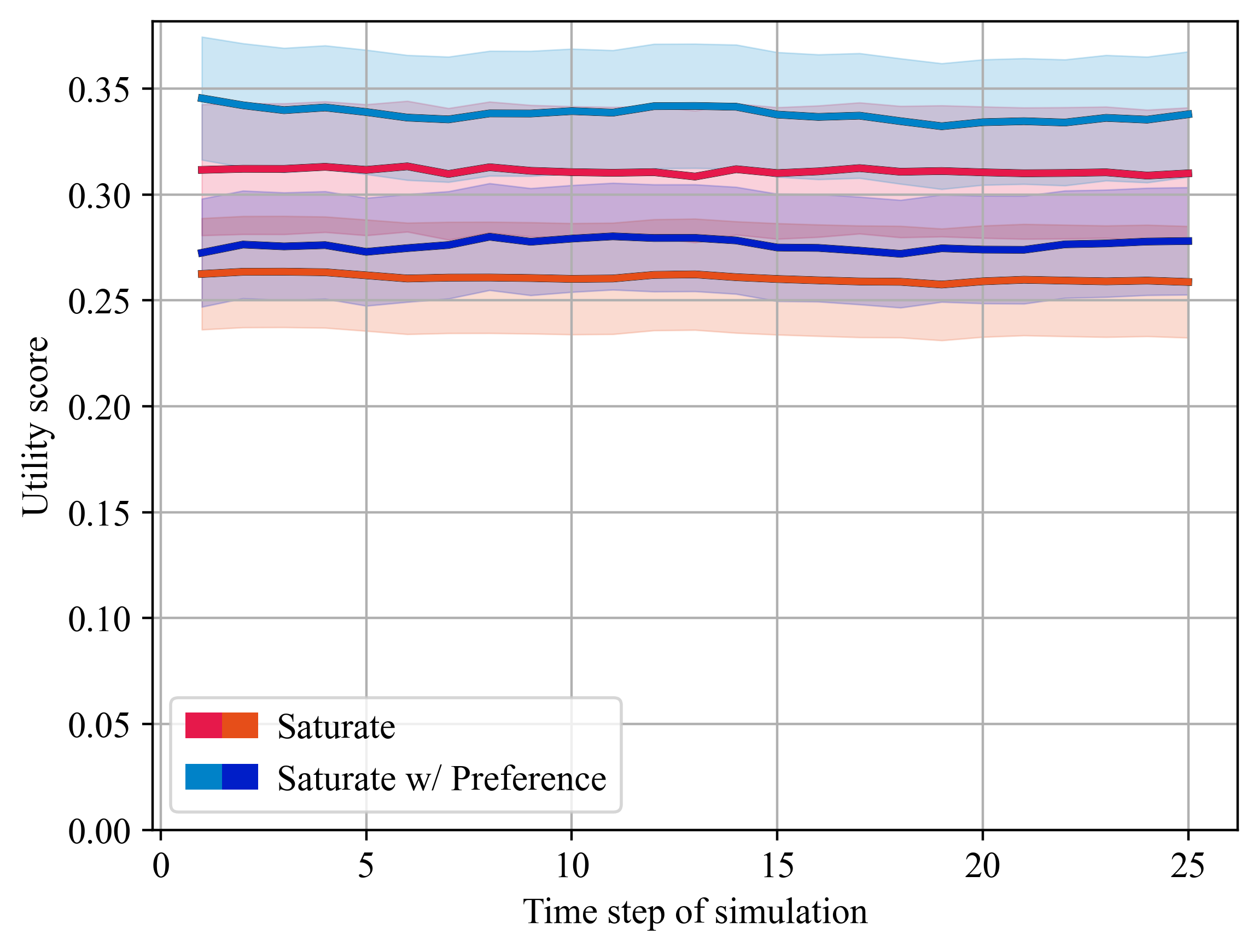}%
    \caption{The average performances over fifteen runs of \textsc{Saturate with Preference} in comparison to SSA on the two objective functions with the highest assigned weight. The highlighted areas indicate one tenth standard deviation. The results have been put through a moving average filter with window size $6$.}%
    \label{fig:swp}%
\end{figure}

\subsection{Application to Online Submodular Optimization}\label{subsec:online}
In this subsection, we assess the performance of \eqref{eq:tr}, which we will call the \textit{time-robust (TR)} formulation for the sake of brevity, as detailed in Section \ref{sec:online}. In summary, the TR formulation aims, using a combination of a momentum-like weighing scheme of the time-varying objective functions in an online setting, along with the idea of relative-entropy regularization, to reuse the same selections made in one step over multiple time steps. In this way, it aims to be more cost-efficient in settings where the selection of more diverse elements incurs additional costs. 

As a baseline, we choose the standard approach of treating each individual objective function $f_t$ observed at time step $t$ as a separate problem and solve it using the \textsc{Stochastic Greedy} algorithm, without any consideration for the conservation of solutions over time steps. Figure \ref{fig:tr} demonstrates the comparison of our proposed approach with the time-window size $t_w =5$ against the standard approach, in terms of the utility achieved, the wall-clock time taken, and the total number of distinct elements used, i.e., $\lvert \bigcup_{t=0}^T \sS_t\rvert$, where $\sS_t$ represents the solution constructed at time step $t$ of the simulation. The results indicate that TR achieves a comparable utility in comparison to the standard approach. The wall-clock time taken fluctuates in the TR formulation, since with a time-window size of $5$, the algorithm only observes during four of every five time steps and plays its solution, performing all function evaluations on every fifth time step, although the total wall-clock time taken is in the same range as the standard approach. However, the total number of distinct elements chosen by the TR formulation is much lower, using less than half the number of distinct elements in comparison to the standard approach.
\begin{figure}
    \centering
    \subfloat[\centering Average utility of the two algorithms]{{\includegraphics[width=.49\linewidth]{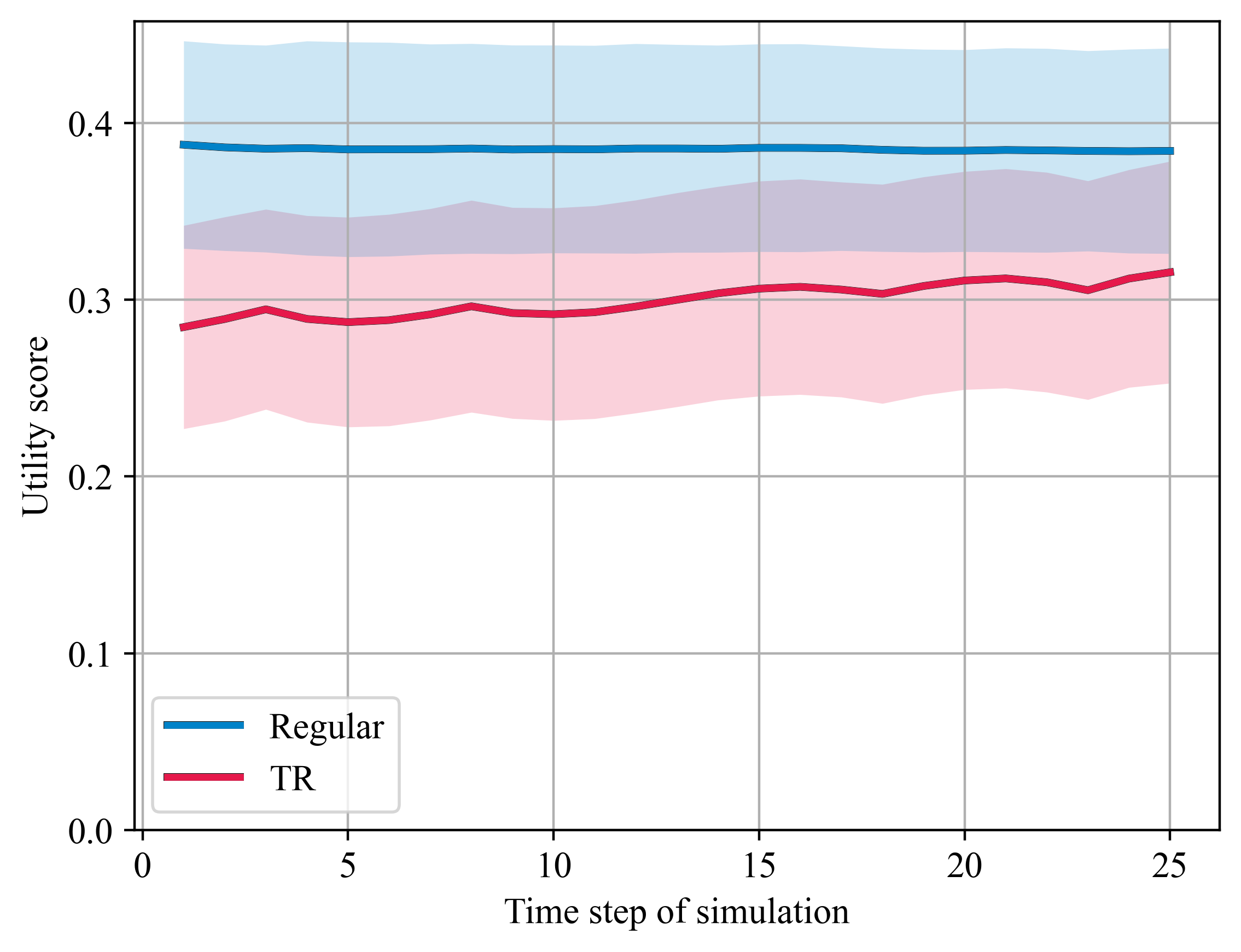} }}%
    %\quad
    \subfloat[\centering Average wall-clock time taken by the two algorithms]{{\includegraphics[width=.49\linewidth]{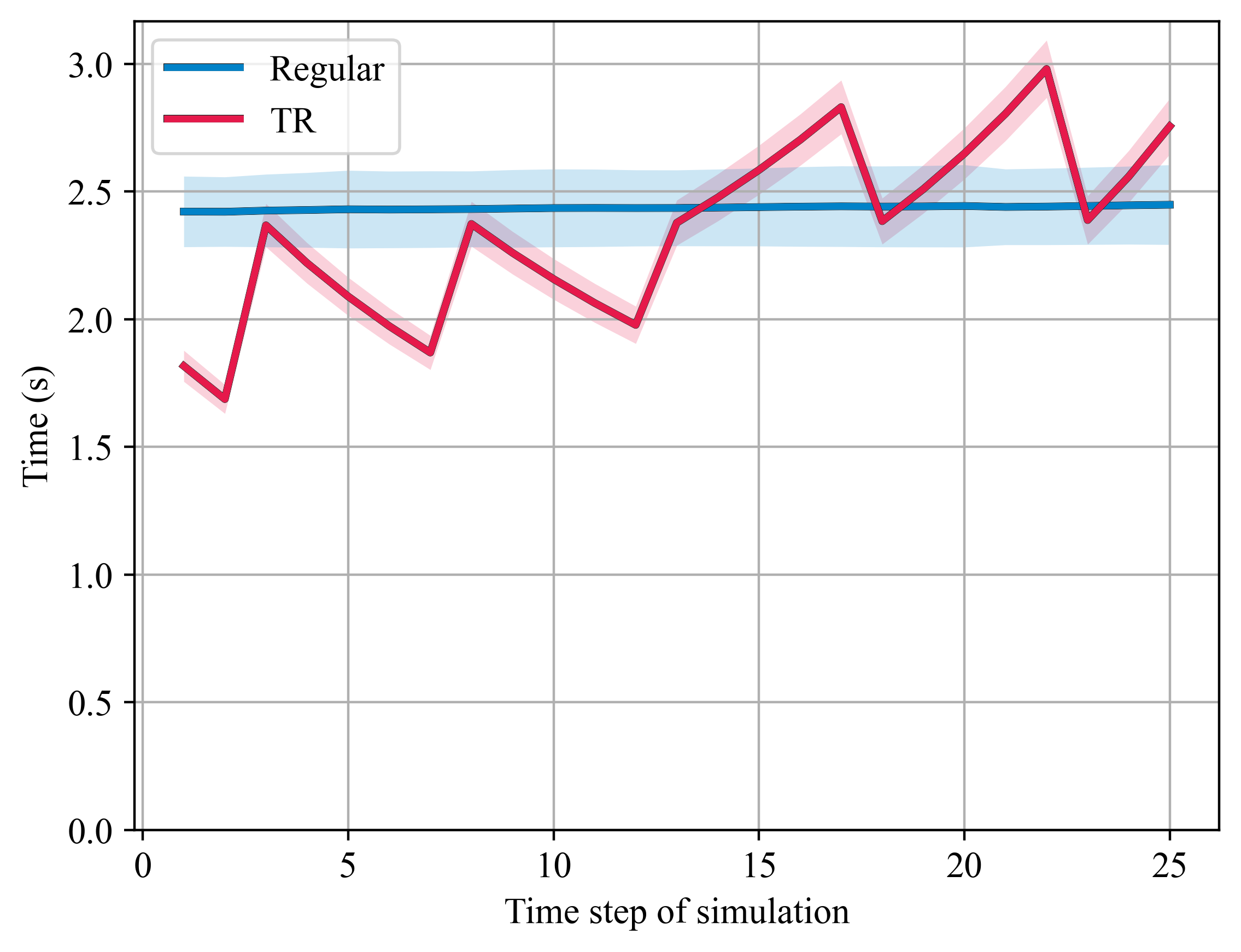} }}%
    \quad
    \subfloat[\centering Average number of distinct elements chosen by the two algorithms]{{\includegraphics[width=.49\linewidth]{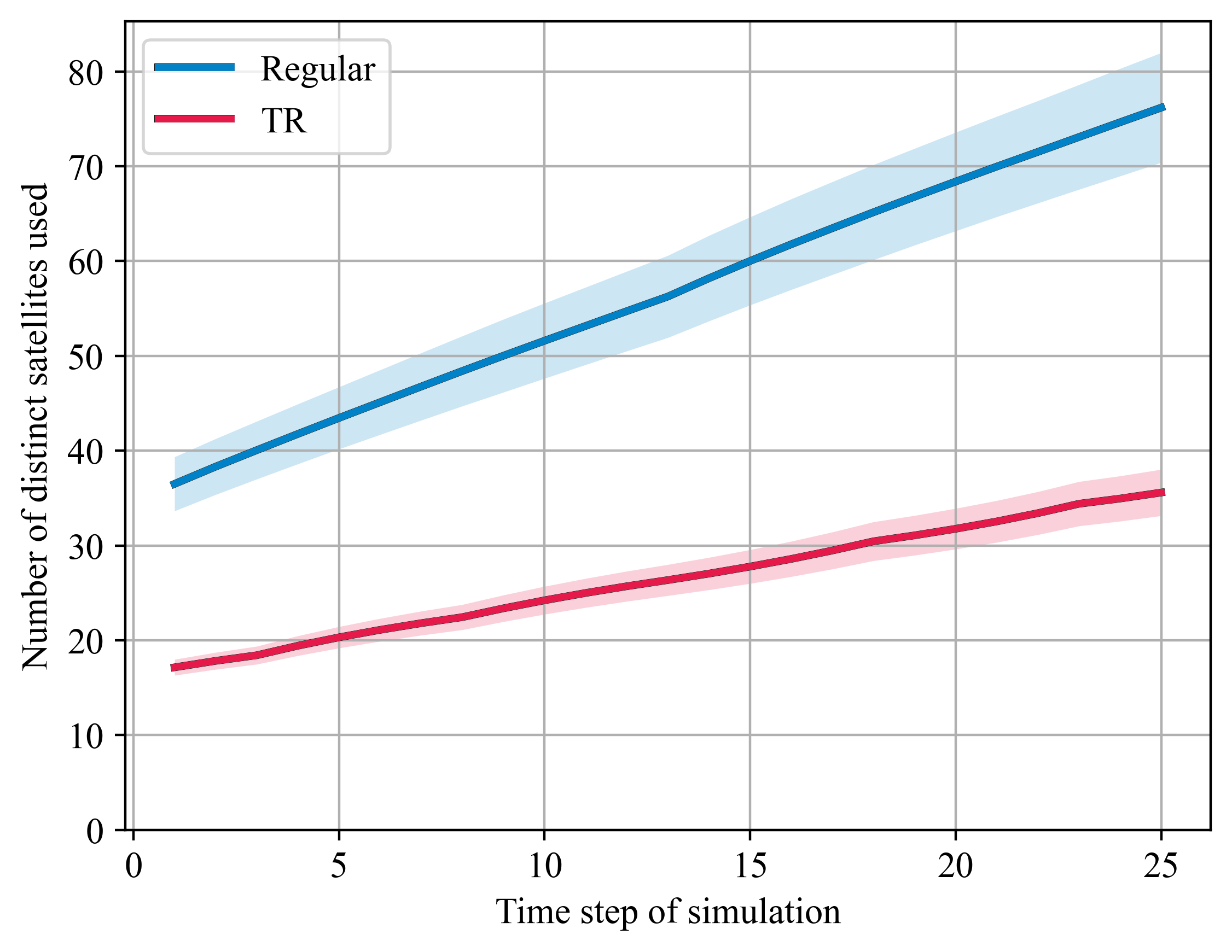} }}%
    \caption{The average performances over $15$ runs of the two algorithms in the online task detailed in Section \ref{subsec:online}. The comparison criteria are the utility achieved, the wall-clock time taken, and the total number of distinct elements used. \textbf{Regular} represents the standard approach of solving each objective function $f_t$ observed at time step $t$ independently of the previous objective functions observed, using the \textsc{Stochastic Greedy} algorithm. \textbf{TR} represents the time-robust formulation detailed in Section \ref{sec:online}, with a time-window size $t_w = 5$. The highlighted areas indicate one standard deviation. The results have been put through a moving average filter with size $6$.}%
    \label{fig:tr}%
\end{figure}
\subsection{Practical Application: Image Summarization}\label{subsec:imgsum}

As a final demonstration using a more up-to-date application involving machine learning and signal processing, we apply the proposed method to image summarization. We tackle the case detailed in \cite{malherbe2022robustness}, which, in short, involves selecting the most representative $K \leq \lvert \sN\rvert$ images out of a dataset $\sN = [\lvert \sN \rvert]$ of indexed images. More formally, for a selection $\sS \subseteq \sN$ of images, one can evaluate the utility $f_i(\sS)$ of each image $i \in \sN$ through its similarity to the closest image in the selection by:
\begin{equation}
f_i(\sS) = 1 - \min_{e\in \sS} d(i, e),
\end{equation}
where $d(i, e)$ designates some distance between images $i$ and $e$. Replicating the setting in the aforementioned work \cite{malherbe2022robustness}, we use the $\lvert \sN \rvert = 819$-image Pokemon dataset\cite{Pokemon}, using image embeddings calculated with an AlexNet\cite{NEURIPS2018_7e448ed9} trained on the ImageNet dataset\cite{5206848}. We use the normalized cosine distance
\begin{equation}
d(i, e) = \frac{v_i \cdot v_e}{\lVert v_i\rVert \lVert v_e\rVert} + \min_{(e, e^{'}) \in \sN^2}\frac{v_e \cdot v_{e^{'}}}{\lVert v_e\rVert \lVert v_{e^{'}}\rVert}.
\end{equation}
The average results obtained from $15$ runs of the algorithms, where in each the reference distribution $Q$ is the uniform distribution, are demonstrated in Figure \ref{fig:imgresults}. We observe that our proposed Local algorithm outperforms SSA for nearly all values of the cardinality bound $K$ on the reference distribution and on the local worst-case distribution, while also taking significantly less computation time, as evidenced by the results of the previous experiments. Regarding the performance on the worst-case task, SSA only outperforms our algorithm within the cardinality constraint range of $K=1$ to $K=5$, but then both algorithms show virtually the same performance, again, with the Local algorithm being much less computationally expensive.
\begin{figure}%
    \centering
    \subfloat[\centering Performance on reference distribution]{{\includegraphics[width=.49\linewidth]{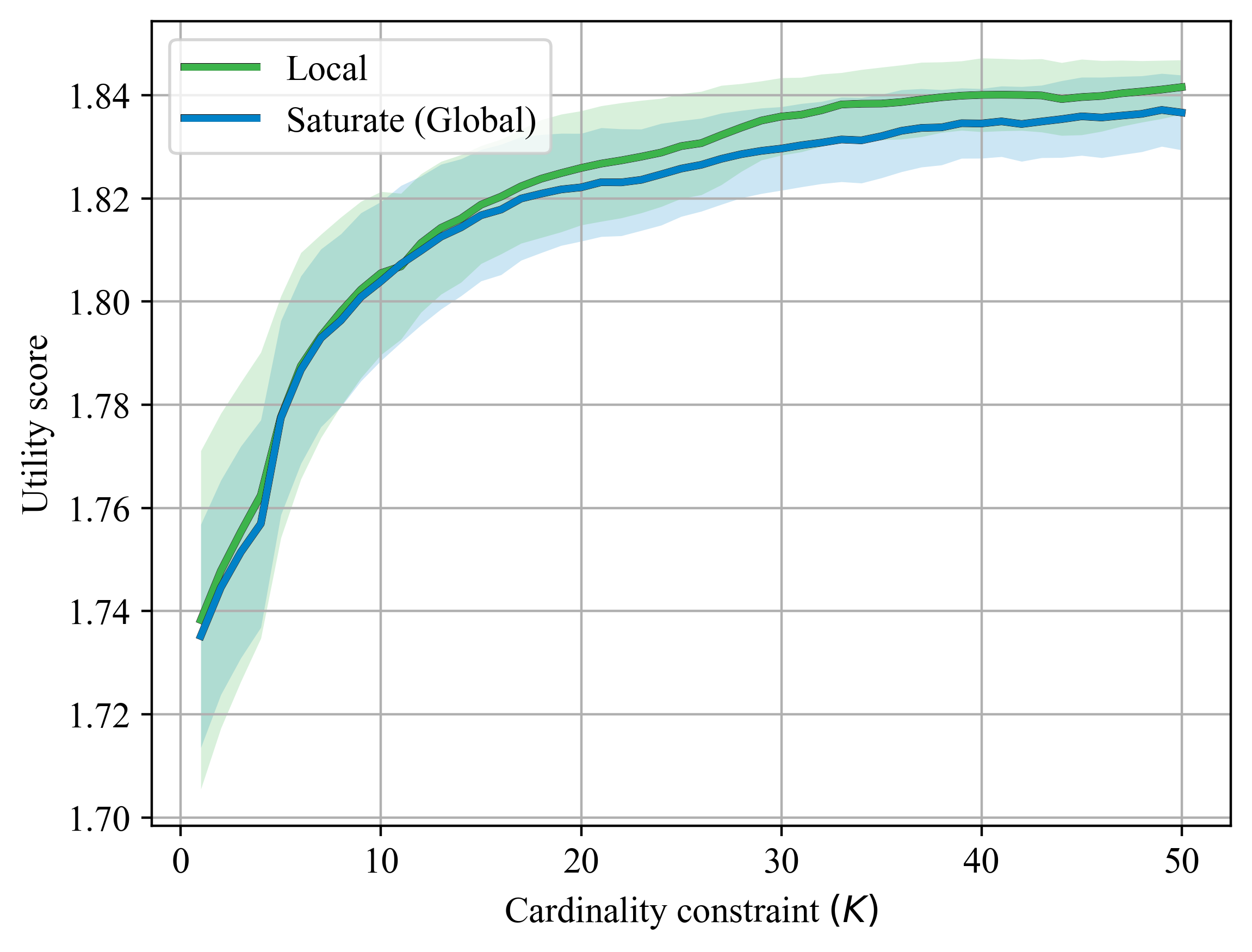} }}%
    \subfloat[\centering Performance on worst-case task]{{\includegraphics[width=.49\linewidth]{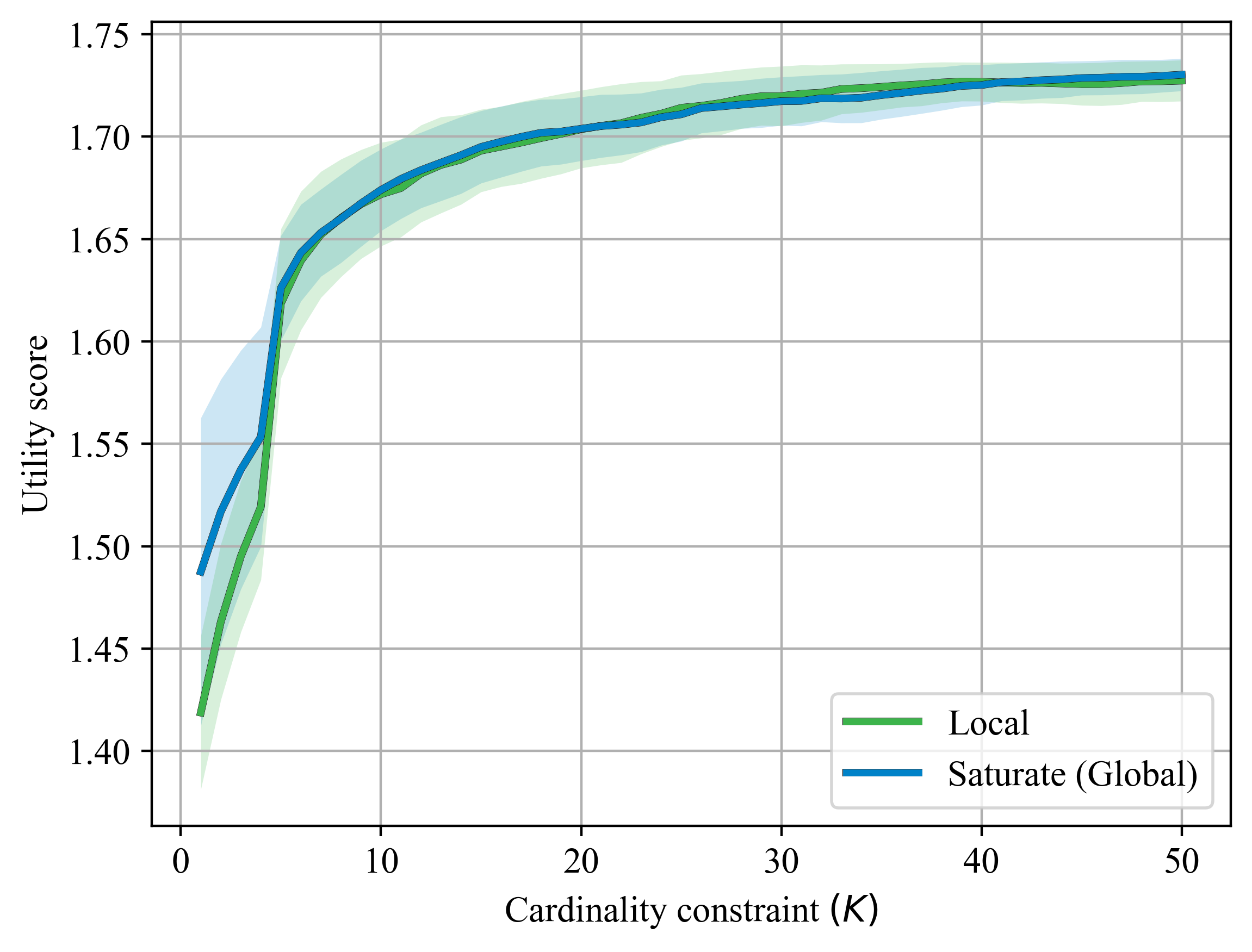}}}%
    \quad
    \subfloat[\centering Performance on local worst-case distribution]{{\includegraphics[width=.49\linewidth]{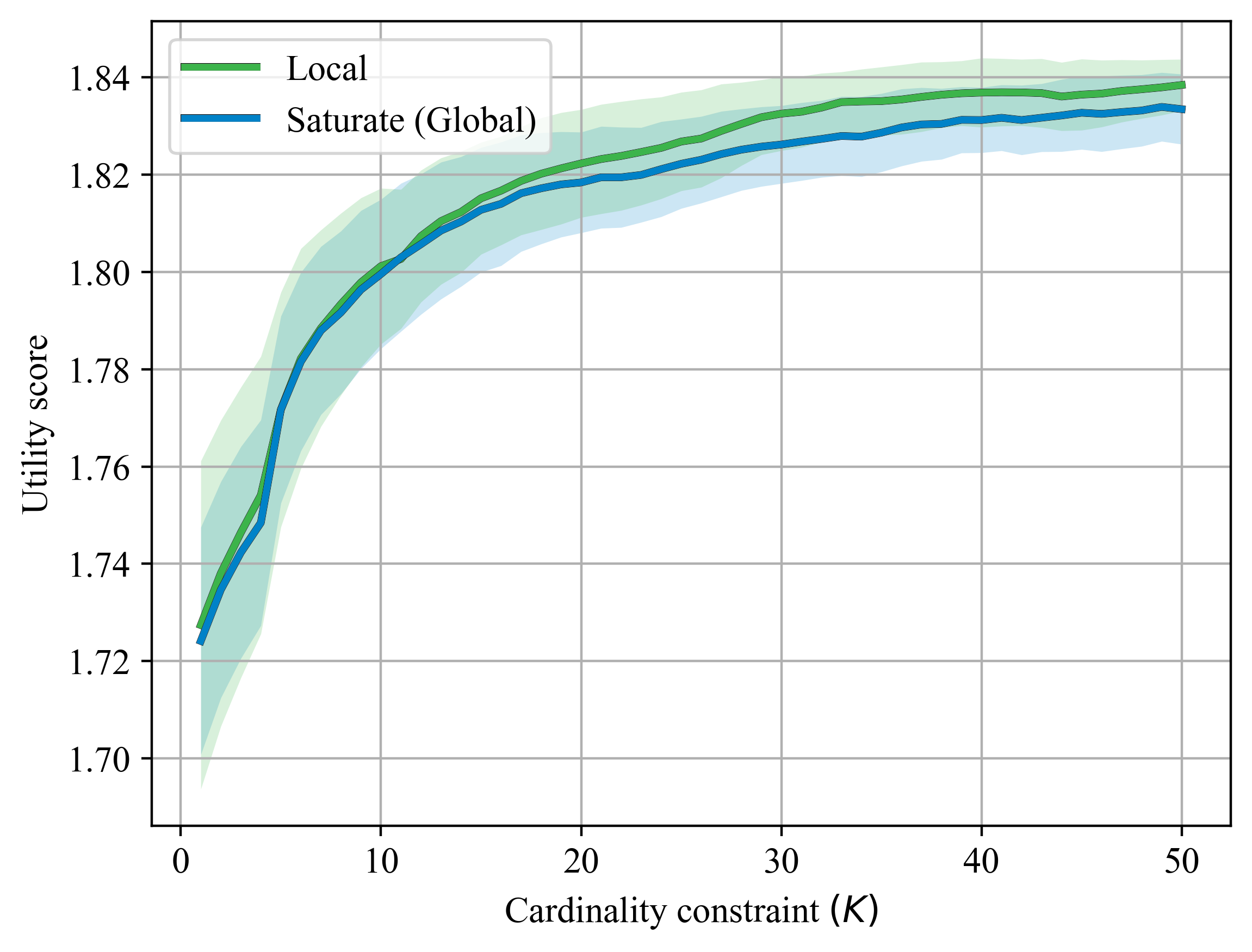}}}%
    \quad
    \caption{The average performances over $15$ runs of the two algorithms focused on optimizing the global worst-case task, and the local worst-case distribution as guided by the reference distribution, evaluated on the three criteria of reference distribution performance, worst-case task performance, and local worst-case distribution performance in the image summarization task of Section \ref{subsec:imgsum}. \textbf{Local} represents the relative-entropy-regularized \textsc{Stochastic Greedy} algorithm solving our novel formulation, aiming for \textit{local} worst-case distributional robustness in the neighborhood of the reference distribution. \textbf{Saturate (Global)} represents SSA, aiming for global worst-case task robustness. The highlighted areas indicate one standard deviation. The results have been put through a moving average filter with window size $6$.}%
    \label{fig:imgresults}%
\end{figure}

\section{Conclusion}
We proposed the novel formulation of \eqref{eq:novel-objective} to find a solution that is \textit{locally} distributionally robust in the neighborhood of a reference distribution which assigns an importance score to each task in a multi-task objective, using various statistical distances as a regularizer. The $L^1$ and $L^\infty$ metrics led to the proposal of the \textsc{Saturate with Preference} algorithm, which incorporates an element of preference into the standard SSA. Afterward, we demonstrated that using relative entropy as a regularizer, through duality, we can show that this novel formulation is equivalent to \eqref{eq:novel-clean}. Then, we proved that this dual formulation gives rise to the composition of a monotone increasing function with a normalized, monotone nondecreasing submodular function, which can be optimized with standard methods such as the \textsc{Stochastic Greedy} algorithm, enjoying theoretical guarantees. We proposed an application of the relative-entropy-regularized objective to online submodular optimization, through the use of a momentum-like weighing of the objective functions observed over time steps. We then experimentally corroborated our results for all three of the proposed settings, motivated by a practical scenario involving a sensor selection problem within a simulation of LEO satellites, using weak submodular functions. For the relative-entropy regularization setting, we compared our algorithm with two other algorithms focused on optimizing the performance of the worst-case task, and on directly optimizing the performance reference distribution itself. We concluded that solving our novel formulation produces a solution that performs well on the reference distribution, is \textit{locally} distributionally robust, and is cheap in terms of computation time. For the \textsc{Saturate with Preference} setting, we showed that our algorithm consistently outperforms the standard SSA in terms of the performance on the objective functions with the highest assigned preference. For the application to the online submodular optimization setting, we demonstrated that our algorithm achieves a comparable utility to the regular method of solving each observed objective function at the moment of its observation, with a comparable wall-clock time taken. However, it does so by using a much smaller total number of distinct satellites. Finally, for a more general, real-life application of the proposed relative-entropy-regularized algorithm, we tackled an image summarization task based on contemporary neural network usage.
\newpage


%\bibliographystyle{IEEEtran}
\bibliographystyle{ieeetr}
\bibliography{refs}

% \begin{IEEEbiography}[{\includegraphics[width=1in,height=1.25in,clip,keepaspectratio]{5x5.jpg}}]{Ege C. Kaya} received the B.S. degree in computer engineering and the B.S. degree in mathematics from Bogazici University, Istanbul, Turkey, in 2022. He is currently pursuing the Ph.D. degree in electrical and computer engineering at Purdue University, West Lafayette, IN, USA. His research interests include optimization for deep learning and submodular optimization.
% \end{IEEEbiography}
% \begin{IEEEbiography}[{\includegraphics[width=1in,height=1.25in,clip,keepaspectratio]{1952-1.jpg}}]{Abolfazl Hashemi} (Member, IEEE) received the B.Sc. degree in electrical engineering from the Sharif University of Technology, Iran, in July 2014, and the M.S.E. and Ph.D. degrees in  electrical and computer engineering from the University of Texas at Austin, USA, in May 2016 and August 2020, respectively. From August 2020 to August 2021 he was a Postdoctoral Scholar at the Oden Institute for Computational Engineering and Sciences at the University of Texas at Austin. Since August 2021, he has been an Assistant Professor at the Elmore Family School of Electrical and Computer Engineering at Purdue University. Abolfazl was the 2019 Schmidt Science Fellows Award nominee from UT Austin, the recipient of the Iranian National Elite Foundation Fellowship, and a Best Student Paper award finalist at the 2018 American Control Conference. His research interests include Large-Scale Optimization for Machine Learning, Signal Processing, and Control.
% \end{IEEEbiography}
\end{document}